\documentclass[review, 11pt]{elsarticle}

\usepackage{lineno,hyperref}

\journal{European Journal of Operations Research}

\biboptions{authoryear}

\usepackage[english]{babel}
\usepackage{amsmath,babel}
\usepackage{amssymb,latexsym}
\usepackage{amstext}
\usepackage{graphicx}
\usepackage{xcolor}
\usepackage{epsfig}
\usepackage{amsfonts}
\usepackage{float} 
\usepackage{multicol} 
\usepackage{multirow} 
\usepackage{mathtools} 
\DeclarePairedDelimiter\abs{\lvert}{\rvert}%
\usepackage{url}
\DeclareMathOperator*{\argmax}{arg\,max} 

\numberwithin{table}{section} 
\numberwithin{figure}{section} 

\newtheorem{theorem}{Theorem}

\newtheorem{definition}[theorem]{Definition}

\newenvironment{proof}[1][Proof]{\noindent\textbf{#1.} }{$\Box$\\}

\newpageafter{author}
\begin{document}

\begin{frontmatter}

\title{Assessment of the influence of features on a classification problem: an application to COVID-19 patients}

%
%
%
%
%

\author[mymainaddress]{Laura Davila-Pena\corref{mycorrespondingauthor}}
\cortext[mycorrespondingauthor]{Corresponding author}
\ead{lauradavila.pena@usc.es}

\author[mysecondaryaddress]{Ignacio García-Jurado}
\ead{ignacio.garcia.jurado@udc.es}

\author[mymainaddress]{Balbina Casas-Méndez}
\ead{balbina.casas.mendez@usc.es}

\address[mymainaddress]{MODESTYA Research Group, Department of Statistics, Mathematical Analysis and Optimisation and IMAT, Faculty of Mathematics, University of Santiago de Compostela, Campus Vida, 15782, Santiago de Compostela, Spain.}
\address[mysecondaryaddress]{MODES Research Group, Department of Mathematics and CITIC, Faculty of Computer Science, University of A Coruña, Campus de Elviña, 15071, A Coruña, Spain.}

\begin{abstract}
This paper deals with an important subject in classification problems addressed by machine learning techniques: the evaluation of the influence of each of the features on the classification of individuals. Specifically, a measure of that influence is introduced using the Shapley value of cooperative games. In addition, an axiomatic characterisation of the proposed measure is provided based on properties of efficiency and balanced contributions. Furthermore, some experiments have been designed in order to validate the appropriate performance of such measure. Finally, the methodology introduced is applied to a sample of COVID-19 patients to study the influence of certain demographic or risk factors on various events of interest related to the evolution of the disease.
\end{abstract}


\end{frontmatter}


\section{Introduction}

A classification problem consists of predicting the value of a qualitative response variable for one or more individuals, making use of the values we know of certain variables (features) of such individuals. Those predictions are based on the knowledge obtained through a training sample of individuals whose values of the features and of the response variable are known.  Classification problems can be addressed by using machine learning techniques. Numerous classifiers have been proposed and analysed in the machine learning literature (see, for example, \citeauthor{FernandezDelgado2014},~ \citeyear{FernandezDelgado2014}). 

In this article we make use of some classification techniques to develop a methodological tool for the exploratory analysis of a training sample of the type described above. Specifically, our objective is to define a sensible measure to estimate the influence of the features on the value of the response variable. Below we illustrate our objective with a real problem of applied research that we recently faced. 

During the first wave of COVID-19 in Spain we had access to a database of 10,454 patients from Galicia (a region in the northwest of Spain) infected with COVID-19 from March 6, 2020 to May 7, 2020. Knowing the characteristics of individuals that significantly increase their probability of needing access to certain health infrastructures is highly useful for health authorities to make the right decisions. Therefore, we set out to use these data to find out which were the values of the features that most influenced the worsening of an infected patient's condition, so that he or she had to be hospitalised, had to be admitted to the ICU or even died.

The problem of studying the influence of features on the values of the response variable that we tackle in this paper has been treated with several differentiating aspects in other works from the  literature. For instance, \cite{Ghaddar2018} introduce an iterative approach to address feature selection in classification using support vector machines and apply it to a case of medical tumours diagnosis. In a sense, the selection of features is a problem prior to the study of the influences we discuss here, because we start with an already selected set of features and then comparatively study their influences.

In the context of classification, \cite{Kononenko2010} introduce a general procedure to assess the importance that the various features have had in the classification of a particular individual. Our approach is different because it is not locally oriented: we do not attempt to evaluate the influence of each feature on the classification of a particular individual, but rather to evaluate the influence of each feature value on the response variable. 

Probably the closest paper to the subject of our research is \cite{Datta2015}. In that paper, the authors also study how influential are the various features in a classification problem. They theoretically base their measure of influence in the binary case, that is, when both the features and the response variable take only two possible values. However, their measure of influence can also be used in the general non-binary case. Another difference with our approach is that they start from a set of observed cases of the feature vectors and an already fixed classifier, and study the influence of each feature for that classifier. In our approach we start from a training sample of individuals for whom we have observed their values of the features and of the response variable; we intend to know the influence of the feature values on the response in the population from which the training sample has been drawn. It is certainly possible to use the approach of \cite{Datta2015} to address our problem: train a classifier with the training sample, and then apply Datta et al.'s measure of influence. In fact, in Section~\ref{sec:experiments} we compared the latter approach with our own.

A common point of \cite{Kononenko2010}, \cite{Datta2015} and our work is that all three make extensive use of cooperative game theory tools, specially the Shapley value. The Shapley value  (\citeauthor{Shapley1953}, \citeyear{Shapley1953}) is a rule for distributing the profits generated by a collection of cooperating agents and it has multiple applications in very diverse fields: just to give a few instances, \cite{Liu2020} use the Shapley value for water resource allocation in multinational river basins, \cite{Saavedra2020} propose a new quota system for the milk market that is based again on the Shapley value, \cite{Li2020} make use of the Shapley value in their study of alliance formation in an assembly system where several upstream complementary suppliers produce components and sell them to a downstream manufacturer. \cite{Algaba2019} is a recent review of the Shapley value, its variants, and its applications.

The organisation of this paper is as follows. Section~\ref{sec:theory} presents the influence measure and discusses its theoretical basis, including an axiomatic characterisation. In Section~\ref{sec:experiments} various experiments are carried out to validate in practice the behaviour of our measure, which is also compared with another approach from the literature. Section~\ref{sec:covid} uses the measure to explore data from a sample of COVID-19 patients to detect features that affect mortality, ICU admission, and patient hospitalisation, and to evaluate the influence of such features. Finally, Section~\ref{sec:conclusions} summarises the main conclusions of this work.

\section{Assessing Influence in Classification} \label{sec:theory}

We start this section by formally establishing what we mean by {\em classification problem}. In one such problem we have a vector of features $X=(X_1,\dots ,X_k)$ and a response variable $Y$.  
$K=\{1,\dots,k\}$ denotes the set of indices of the features.
Each feature $X_j$ takes values in a finite set ${\cal A}_j$ and $Y$ takes values in a finite set ${\cal B}$. We also have a training sample ${\cal M}=\{(X^i,Y^i)\}_{i=1}^n$, where $X^i=(X_1^i,\dots ,X_k^i)$ and $Y^i$ are the observed values of the features and the response variable corresponding to individual $i$. A classification problem is thus characterised by a triplet $(X,Y,{\cal M})$.

A {\em classifier} trained with sample ${\cal M}$ is a map $f^{\cal M}$ that assigns to every\phantom{a} $a\in {\cal A}={\cal A}_1\times\dots\times{\cal A}_k$ (an observation of $X$) a probability distribution over ${\cal B}$, i.e., $f^{\cal M}(a)=(f^{\cal M}_b(a))_{b\in{\cal B}}$ with $f^{\cal M}_b(a)\geq0$, for all $b\in{\cal B}$, and $\sum_{b\in{\cal B}}f^{\cal M}_b(a)=1$. Each $f^{\cal M}_b(a)$ is the estimated probability that an individual whose observed values of the features are given by $a$ belongs to group $b$ of the response variable $Y$. From now on, ${\cal A}_V$, $a_V$, $X_V$, and $X^i_V$ will denote the restrictions of ${\cal A}$, $a$, $X$, and $X^i$ to the variables of $V$, respectively (for all $V\subseteq K$). 

Our goal in this section is to use classification techniques to define a measure that allows us to study the influence of the features on the response variable. The formal definition of an influence measure  is the one included below.

\begin{definition}\label{def:infl_measure}
	An {\em influence measure} for $(X,Y,{\cal M})$ is a map $I$ that assigns to every $a_R\in {\cal A}_R$ ($R\subseteq K$), $b\in{\cal B}$, and $T\subseteq K$ $(T\neq \emptyset)$ a vector $I(a_R,b,T)=(I_l(a_R,b,T))_{l\in T}\in\mathbb{R}^{T}$. The vector $I(a_R,b,T)$ provides an evaluation of the influence that each feature $X_l$ $(l \in T)$ has on whether the response is worth $b$ when $X_R$ is worth $a_R$ and we only take into account the features $\{X_l\}_{l\in T}$.
\end{definition}

Section~\ref{sec:covid} illustrates the interest of having a sensible influence measure. In this section we introduce and theoretically support one based on the Shapley value of cooperative games. In order to facilitate the reader's understanding, we include the definition of the Shapley value below. First, recall that a cooperative game is a par $(N,v)$, where $N$ is the finite set of players, and $v:2^{N} \rightarrow \mathbb{R}$ is the characteristic function of the game, which satisfies $v(\emptyset)=0$. We usually interpret $v(S)$ as the gain that coalition $S\subseteq N$ can obtain. Also, $G(N)$ represents the set of all cooperative games with set of players $N$. In general, we identify $(N,v)$ with its characteristic function, $v$. An extensively addressed problem in cooperative games is to allocate $v(N)$ among the cooperating agents. One of the most important allocation rules is the Shapley value, $\Phi:G(N) \rightarrow \mathbb{R}^{N}$, which represents a fair compromise for the players and it is defined by the following expression:
\begin{equation*}
	\Phi_{i}(v) = \sum_{S\subseteq N\setminus\{i\}}{\dfrac{\abs{S}!\,(\abs{N}-\abs{S}-1)!}{\abs{N}!}\left(v(S\cup \{i\}) - v(S)\right)},
\end{equation*}
for all $v\in G(N)$ and $i\in N$. For more details on cooperative games see, for instance, \cite{Gonzalez2010}.

 Next, we consider two desirable properties and prove that there exists a unique influence measure that fulfils them: the one based on the Shapley value. The first property takes into account that a measure of influence simply distributes among the $T$ features the total influence that such features have in that the value of the response variable is $b$ when $X_R$ equals $a_R$. One way to estimate that total influence using the classifier $f^{\cal M}$ is given by the following expression:
\begin{equation}\small
	\label{exp1:davilaetal}
	\frac{1}{n^b_{a_R}}\sum_{(X^i,Y^i)\in {\cal M}^b_{a_{R}}}\left(\frac{1}{\abs{{\cal A}_{K\backslash T}}}\sum_{a'_{K\backslash T}\in {\cal A}_{K\backslash T}}f^{{\cal M}}_b(X^i_T,a'_{K\backslash T})-\frac{1}{\abs{{\cal A}}}\sum_{a'\in {\cal A}}f^{{\cal M}}_b(a')\right),
\end{equation}
where ${\cal M}^b_{a_R}$ denotes the subsample of ${\cal M}$ formed by the observations $(X^i,Y^i)$ with $X^i_R=a_R$ and $Y^i=b$, and $n^b_{a_R}$ denotes the size of the subsample ${\cal M}^b_{a_R}$. 

Notice that expression (\ref{exp1:davilaetal}) can be interpreted as an estimation of the variability of the response variable due to the $T$ features (using $f^{\cal M}$). Therefore, the first property we ask for an influence measure is the $f^{\cal M}$-Efficiency below.

\vspace*{0.35cm}
\noindent
{\bf $f^{\cal M}$-Efficiency.} An influence measure $I$ satisfies $f^{\cal M}$-Efficiency if, for every $(X,Y,{\cal M})$, every $a_R\in {\cal A}_R$ ($R\subseteq K$), $b\in {\cal B}$, and $T\subseteq K$ ($T\neq \emptyset$), it holds that $\sum_{l\in T}{I_l(a_R,b,T)}$ is equal to the amount in expression (\ref{exp1:davilaetal}).

\vspace*{0.25cm}

The second property that we consider is a fairness property that treats all features in a balanced way. Informally, it states that given two of these features, the effect of ignoring one to the measure of the influence of the other is identical for both features. Note that the marginal loss or gain of influence that the inclusion or exclusion of one feature causes to another feature is due to the dependency that exists between the two. The fact that the dependence between features is symmetrical, makes advisable the property of balanced contributions.

\vspace*{0.25cm}

\noindent
{\bf Balanced Contributions.} An influence measure satisfies Balanced Contributions if, for every $(X,Y,{\cal M})$, every $a_R\in {\cal A}_R$ ($R\subseteq K$), $b\in {\cal B}$, $T\subseteq K$ ($T\neq \emptyset$), and $l,m\in T$ with $l\neq m$,
$$I_l(a_R,b,T)-I_l(a_R,b,T\backslash \{m\})=I_m(a_R,b,T)-I_m(a_R,b,T\backslash \{l\}).$$


Now we state and prove the main mathematical result of this section. It provides a characterisation and a formal expression of an influence measure that satisfies all the properties introduced above.

\begin{theorem}\label{th:infl_measure}
	There exists a unique influence measure for $(X,Y,{\cal M})$ which satisfies the properties of $f^{\cal M}$-Efficiency and Balanced Contributions. For all $a_R\in {\cal A}_R$ ($R\subseteq K$), $b\in {\cal B}$, $T\subseteq K$ $(T\neq \emptyset)$ and $l\in T$, this measure (that we denote by  $I^{\Phi}$) is given by
	\begin{equation}\label{eq:influence_measure}
		I_l^{\Phi}(a_R,b,T)=\frac{1}{n^b_{a_R}}\sum_{(X^i,Y^i)\in {\cal M}^b_{a_R}}\Phi_l(v^b_{X^i}|_T),
	\end{equation}
	where $\Phi$  denotes the Shapley value,
	$v^b_{X^i}$ denotes the game with set of players $K$ given by
	\begin{equation}
		v^b_{X^i}(S)=\frac{1}{\abs{{\cal A}_{K\backslash S}}}\sum_{a'_{K\backslash S}\in {\cal A}_{K\backslash S}}f^{{\cal M}}_b(X^i_S,a'_{K\backslash S})-\frac{1}{\abs{{\cal A}}}\sum_{a'\in {\cal A}}f^{{\cal M}}_b(a'),
		\label{eg:game}
	\end{equation}
	for all $S\subseteq K$,
	and $v^b_{X^i}|_T$ denotes the restriction of the game $v^b_{X^i}$ to the subsets of $T$.\footnote{The game in $(\ref{eg:game})$ results to be the same as the one used in \cite{Kononenko2010} to assess the importance of the various features in the classification of a particular individual in a classification problem.} 
\end{theorem}


\begin{proof}
	{\it Existence.}
	To show that $I^{\Phi}$ satisfies $f^{\cal M}$-Efficiency, take $a_R\in {\cal A}_R$ ($R\subseteq K$), $b\in {\cal B}$, and $T\subseteq K$ $(T\neq \emptyset)$. \cite{Shapley1953} proves that the Shapley value of cooperative games satisfies an efficiency property. In our case, this property implies that
	$$\sum_{l\in T}{\Phi_l(v^b_{X^i}|_T)} = v^b_{X^i}(T).$$
	Applying this result we obtain that:
	\begin{equation*} \small
		\begin{split}
		& \sum_{l\in T}{I_l^{\Phi}(a_R,b,T)}\\
		= \: & \sum_{l\in T}{\frac{1}{n^b_{a_R}}\sum_{(X^i,Y^i)\in {\cal M}^{b}_{a_R}}{\Phi_l(v^b_{X^i}|_T)}}\\  
		= \: & \frac{1}{n^b_{a_R}}\sum_{(X^i,Y^i)\in {\cal M}^{b}_{a_R}}\sum_{l\in T}{\Phi_l(v^b_{X^i}|_T)}\\  
		= \: & \frac{1}{n^b_{a_R}}\sum_{(X^i,Y^i)\in {\cal M}^b_{a_R}}v^b_{X^i}(T)\\
		= \: & \frac{1}{n^b_{a_R}}\sum_{(X^i,Y^i)\in {\cal M}^b_{a_R}}\left(\frac{1}{\abs{{\cal A}_{K\backslash T}}}\sum_{a'_{K\backslash T}\in {\cal A}_{K\backslash T}}f^{{\cal M}}_b(X^i_T,a'_{K\backslash T})-\frac{1}{\abs{{\cal A}}}\sum_{a'\in {\cal A}}f^{{\cal M}}_b(a')\right).
		\end{split}
	\end{equation*}

	To show that $I^{\Phi}$ satisfies Balanced Contributions, let $a_R\in {\cal A}_R$ ($R\subseteq K$), $b\in {\cal B}$, $T\subseteq K$ $(T\neq \emptyset)$, and $l,m\in T$ with $l\neq m$. \cite{Myerson1980} proves that the Shapley value of cooperative games satisfies a property of balanced contributions. In our case, this property implies that 
	$$ \Phi_{l}(v^{b}_{X^i}|_T) - \Phi_{l}(v^{b}_{X^i}|_{T\backslash \{m\}}) = \Phi_{m}(v^{b}_{X^i}|_T) - \Phi_{m}(v^{b}_{X^i}|_{T\backslash \{l\}}).$$
	Applying this result we obtain that:	

\begin{equation*}
	\begin{split}
		& I_l^{\Phi}(a_R,b,T)-I_l^{\Phi}(a_R,b,T\backslash \{m\})\\
		= \: & \frac{1}{n^b_{a_R}}\sum_{(X^i,Y^i)\in {\cal M}^b_{a_R}}\Phi_l(v^b_{X^i}|_T)-\frac{1}{n^b_{a_R}}\sum_{(X^i,Y^i)\in {\cal M}^b_{a_R}}\Phi_l(v^b_{X^i}|_{T\backslash \{m\}})\\
		= \: &  \frac{1}{n^b_{a_R}}\sum_{(X^i,Y^i)\in {\cal M}^b_{a_R}}\left(\Phi_l(v^b_{X^i}|_T)-\Phi_l(v^b_{X^i}|_{T\backslash \{m\}})\right)\\
		= \: &  \frac{1}{n^b_{a_R}}\sum_{(X^i,Y^i)\in {\cal M}^b_{a_R}}\left(\Phi_m(v^b_{X^i}|_T)-\Phi_m(v^b_{X^i}|_{T\backslash \{l\}})\right)\\
		= \: &  \frac{1}{n^b_{a_R}}\sum_{(X^i,Y^i)\in {\cal M}^b_{a_R}}\Phi_m(v^b_{X^i}|_T)-\frac{1}{n^b_{a_R}}\sum_{(X^i,Y^i)\in {\cal M}^b_{a_R}}\Phi_m(v^b_{X^i}|_{T\backslash \{l\}})\\
		= \: & I_m^{\Phi}(a_R,b,T)-I_m^{\Phi}(a_R,b,T\backslash \{l\}).
	\end{split}
\end{equation*}
	\noindent
	{\it Uniqueness.}
	We show uniqueness by induction on the size of $T$. Suppose that $I^1$ and $I^2$ are two influence measures satisfying $f^{\cal M}$-Efficiency and Balanced Contributions.
	If $\abs{T} =1$, by $f^{\cal M}$-Efficiency, 
	$$I^1(a_R,b,T)=\frac{1}{n^b_{a_R}}\sum_{(X^i,Y^i)\in {\cal M}^b_{a_R}}v^b_{X^i}(T)=I^2(a_R,b,T).$$
	{Assume now that $I^1(a_R,b,S) = I^2(a_R,b,S)$ for all $S\subseteq T$ with $1\leq |S|<|T|$. Then by Balanced Contributions, for all $l,m\in T$, $l\neq m$,
		\begin{equation}
			I_l^1(a_R,b,T)-I^1_m(a_R,b,T)=I_l^2(a_R,b,T)-I^2_m(a_R,b,T).
			\label{eq:1}
		\end{equation}
		Using $f^{\cal M}$-Efficiency,
		\begin{equation}
			\sum_{l\in T}{I_l^1(a_R,b,T)}=\sum_{l\in T}{I^2_l(a_R,b,T)}.
			\label{eq:2}
		\end{equation}
	\vspace{-0.1cm}
		By (\ref{eq:1}) and (\ref{eq:2}) it is obtained that: \vspace{-0.1cm}
		$$I^1_l(a_R,b,T)=I^2_l(a_R,b,T)\ {\rm for}\ {\rm all}\ l\in T.$$
		This last expression gives the uniqueness.}
\end{proof}

\section{Empirical results} \label{sec:experiments}

In this section we show the performance of the proposed influence measure~(\ref{eq:influence_measure}) by means of a computational study. Three different experiments have been carried out using the software \texttt{R}. The objective of such simulations is to corroborate that the results obtained by the methodology introduced in the current work are in accordance with the expected ones. Furthermore, these results are compared with those obtained by the influence measure introduced in \cite{Datta2015}, which counts the number of times that a modification in a feature results in a different classification. We provide the formal definition of such an influence measure below.

\begin{definition}\label{def:infl_measure_datta}
	Given a training set $\mathcal{M} = \{(X^{i},Y^{i})\}_{i=1}^{n}$ and a classifier $f^{\mathcal{M}}$, the influence of the $j$-th feature is  \small
	\begin{equation*}
		\chi_{j}(f^{\mathcal{M}}) = \sum_{a'\in \{X^{i}\}}{\sum_{\substack{a_{j}\in  \mathcal{A}_{j}: \\ (a'_{-j},a_{j})\in \{X^{i}\}}}\min\left\{{\left|\argmax_{b\in \mathcal{B}}{ f_{b}^{\mathcal{M}}(a'_{-j},a_{j})}-\argmax_{b\in \mathcal{B}}{f_{b}^{\mathcal{M}}(a')}\right|},1\right\}},
	\end{equation*}
where $\{X^{i}\}$ denotes $\{(X_{1}^{i},\dots,X_{k}^{i})\}_{i=1}^{n}$, and $\mathcal{B}\subset \mathbb{N}$.
	
\end{definition}

The classifier used in this paper is Breiman's random forest classifier (\citeauthor{Breiman2001}, \citeyear{Breiman2001}), implemented in Weka\footnote{\url{http://www.cs.waikato.ac.nz/ml/weka}.} and used through \texttt{RWeka}\footnote{\url{https://cran.r-project.org/web/packages/RWeka/index.html}.}. This choice is motivated by the excellent result of the random forest type classifiers (see, for example, \citeauthor{FernandezDelgado2014}, \citeyear{FernandezDelgado2014}). The code was run on a quad-core Intel~i7-8665U CPU with 16GB RAM.

The procedure adopted in the experiments is as follows. We start from a sample of individuals from which their attributes and response are known, ${\cal M}=\{(X^i,Y^i)\}_{i=1}^n$. Right after, such sample is used to train a previously chosen classifier, obtaining $f^{\cal M}$.
To evaluate the influence of feature $X_j$ on the response $Y$ taking the value $b$, the quantities $I^{\Phi}_{j}(a_j,b,K)$ and $\sum_{l\in K}{I_{l}^{\Phi}(a_j,b,K)}$ are computed and analysed for all $a_j\in {\cal A}_j$.

For the first experiment, a sample of $1000$ instances with four binary features $\{X_{1}, X_{2}, X_{3}, X_{4}\}$ was generated. Such attributes take the values $0$ and $1$ with probability $0.5$ (hence, $a_{j} \in \mathcal{A}_{j}=\{0,1\}, \: j \in K$). In half of the instances, the value of $Y$ coincides with the value of $X_{1}$, while in the remaining instances the value of $Y$ coincides with the value of $X_{2}$; note thus that $b\in \mathcal{B}=\{0,1\}$. The following step is to select those observations whose assigned class was $b=1$. Afterwards, for each attribute $X_{j}, \: j\in K$, and each of its possible values, we study the influence that such feature had on the response when it took such value. Since the procedure by which the class has been generated is known, it is evident that the influence of attributes $X_{3}$ and $X_{4}$ should be independent of their values. Furthermore, the value 1 for features $X_{1}$ and $X_{2}$ should have a stronger influence in the classification than the value 0. Table~\ref{tab:simulation1} and Figure~\ref{fig:simulation1} present the results obtained for this simulation, which took a runtime of $9.3$ minutes.

Indeed, it can be observed that for attributes $X_{1}$ and $X_{2}$ the value $I_{j}^{\Phi}(a_{j},b,K)$ is positive when $a_{j}=1$ and negative when $a_{j}=0$, which means that features $X_{1}$ and $X_{2}$ taking the value $1$ works in favour of the response resulting in $1$, unlike what happens if these features are worth $0$. Note also that  $\sum_{l\in  K}I_{l}^{\Phi}(a_{j},b,K)$ is the total influence of the four features on the response being $1$ when feature $X_{j}$ takes the value $a_{j}$. In view of the results obtained, for features $X_{1}$ and $X_{2}$ the quantities $I_{j}^{\Phi}(a_{j},b,K)$ and  $\sum_{l\in  K}I_{l}^{\Phi}(a_{j},b,K)$   are closer when $a_{j}=1$ than when $a_{j}=0$. Thus, the total influence on the response being $1$ when either $X_{1}$ or $X_{2}$ are $1$, is in fact due to these specific attributes taking the value $1$. In the case of features $X_{3}$ and $X_{4}$, their influence is near $0$ whatever value they take.

\begin{table}[H]
	\begin{center}
		\resizebox{6.6cm}{!}{
		\begin{tabular}{|c|c|c|c|}
			\hline
			$X_{j}, j \in K$ & $a_{j}$ & $\sum_{l\in K}I^{\Phi}_{l}(a_{j},b,K)$ & $I^{\Phi}_{j}(a_{j},b,K)$ \\ 
			\hline \hline
			\multirow{2}{*}{$X_{1}$} & 0 & -0.002 & -0.250\\
			\cline{2-4}
			& 1 & 0.344  & 0.247 \\
			\hline \hline
			\multirow{2}{*}{$X_{2}$} & 0 & -0.019 &  -0.260 \\
			\cline{2-4}
			& 1 & 0.361 & 0.260  \\ 
			\hline \hline
			\multirow{2}{*}{$X_{3}$} & 0 & 0.268  & 0.000 \\
			\cline{2-4}
			& 1 & 0.268 & 0.000\\
			\hline \hline
			\multirow{2}{*}{$X_{4}$} & 0 & 0.258 & -0.010 \\
			\cline{2-4}
			& 1 & 0.277 & 0.010 \\
			\hline
		\end{tabular}
	}
		\caption{Results for simulation 1.}
		\label{tab:simulation1}
	\end{center}
\end{table}

\vspace{-0.7cm}
\noindent
\begin{figure}[H]
\begin{minipage}[c]{6.3cm}
		\begin{center}
			\includegraphics[width=5cm]{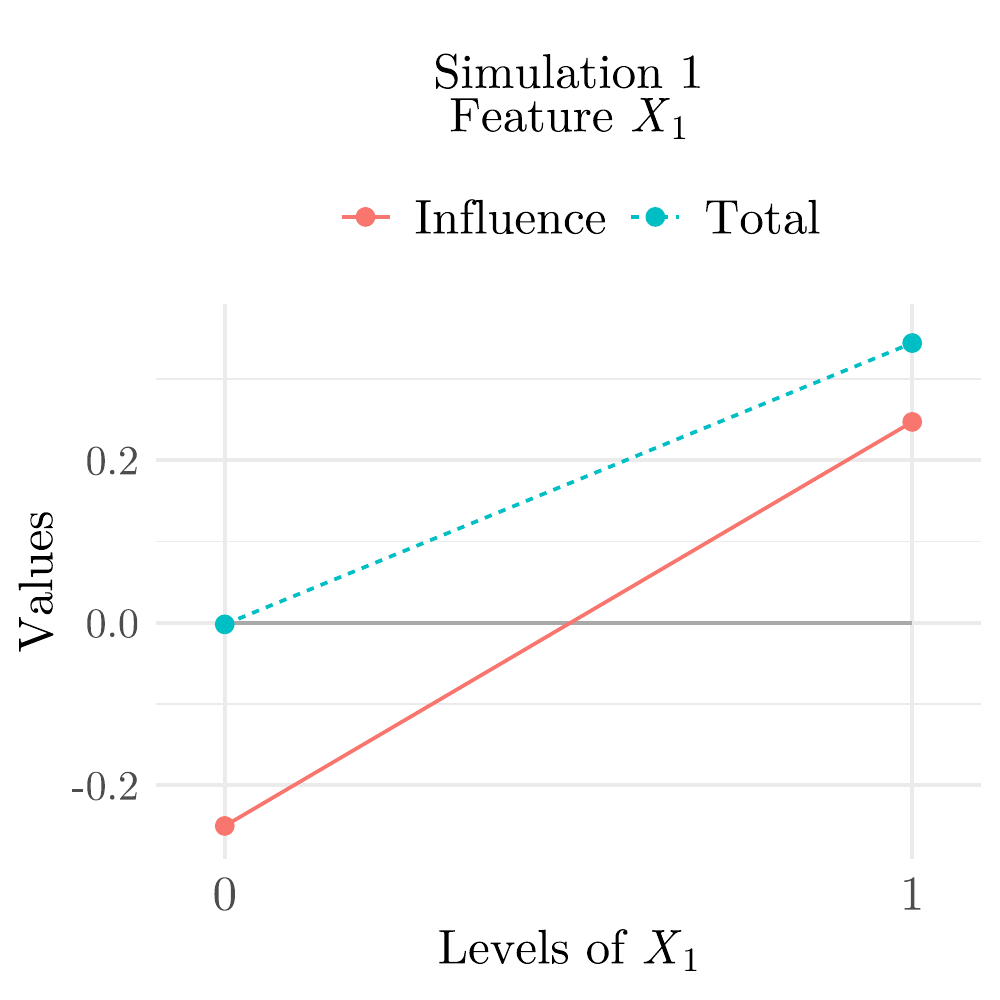}
		\end{center}
\end{minipage}
\hspace{0.1cm}
\begin{minipage}[c]{6.3cm}
		\begin{center}
			\includegraphics[width=5cm]{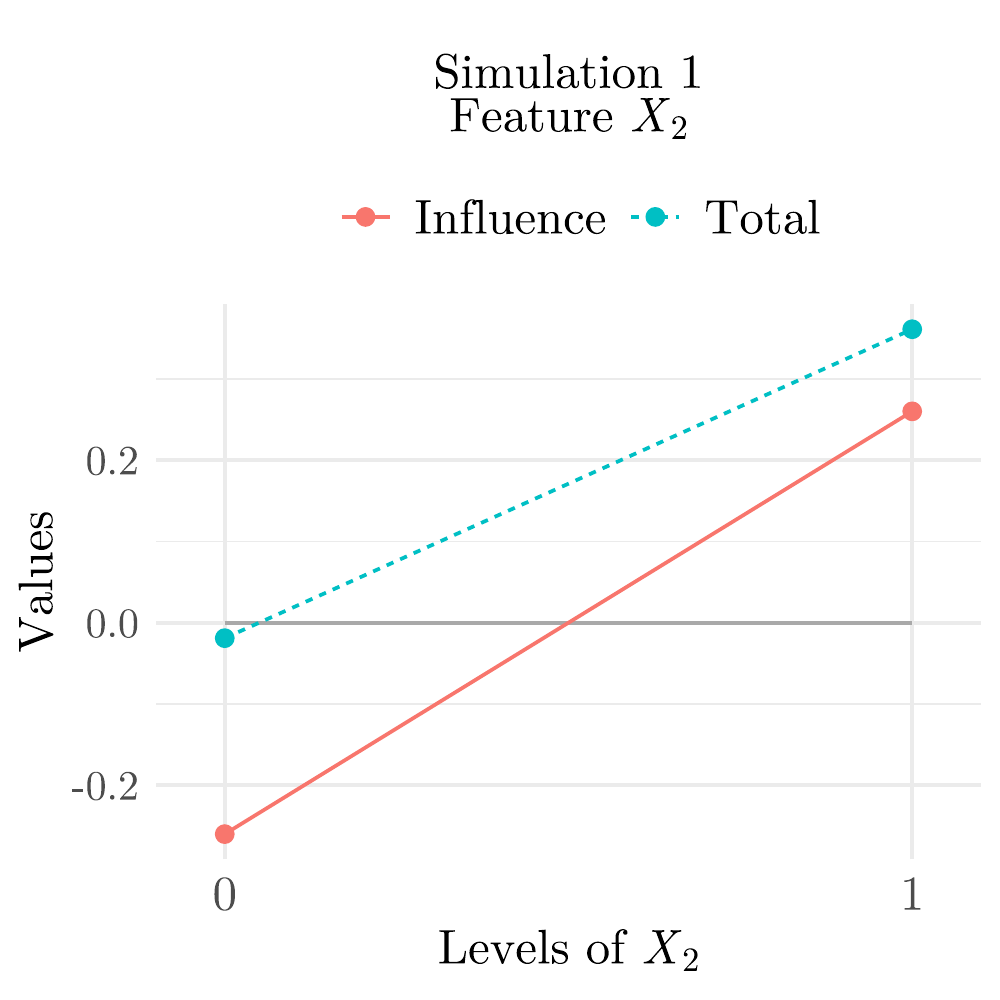}
		\end{center}
\end{minipage}

\noindent
\begin{minipage}[c]{6.3cm}
		\begin{center}
			\includegraphics[width=5cm]{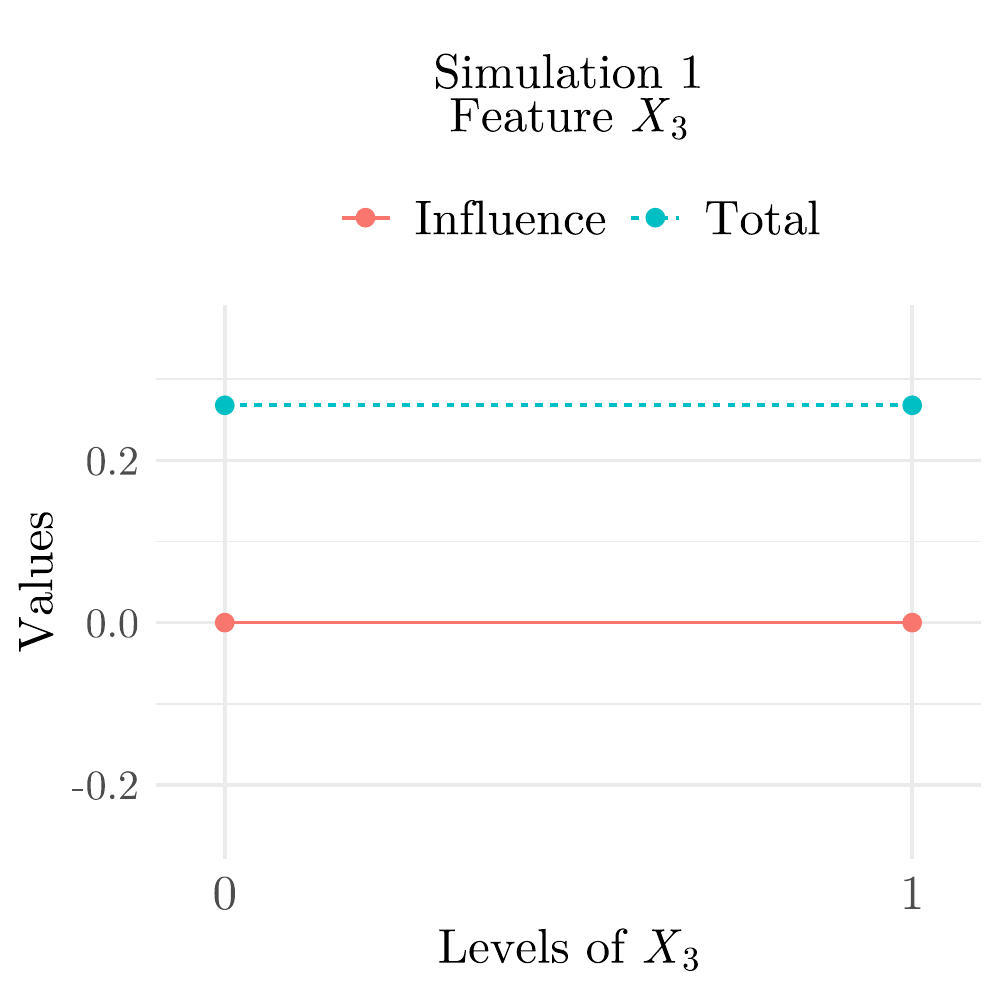}
		\end{center}
\end{minipage}
\hspace{0.2cm}
\begin{minipage}[c]{6.3cm}
		\begin{center}
			\includegraphics[width=5cm]{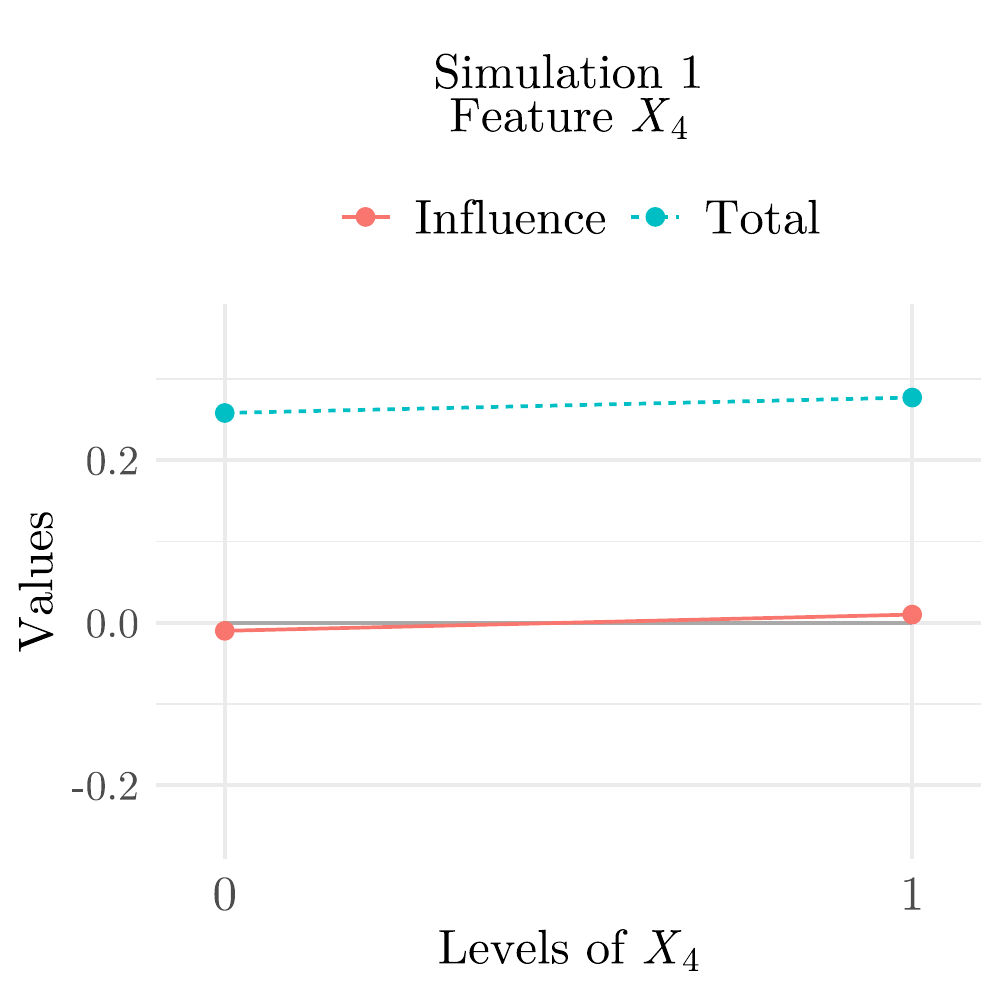}
		\end{center}
\end{minipage}
\caption{Influence and total influence for the features (Simulation 1).}
\label{fig:simulation1}
\end{figure}


Applying the procedure in \cite{Datta2015} to the previous experiment, we obtain the measure $(0.50, 0.50, 0.25, 0.25)$. As expected, features $X_{1}$ and $X_{2}$ present a higher influence than $X_{3}$ and $X_{4}$. Just as we have already mentioned, Datta et al.'s procedure measures the number of times that a change in a specific attribute produces a different response. Thus, it only takes positive values, which prevents us from knowing the direction of the influence. In our case, setting features $X_{1}$ and $X_{2}$ to $0$ works against the response being $1$, and this is made clear by the negative sign of their influences.

The second experiment differs from the previous one in the procedure to assign the class to the instances. The response is now generated as a binary vector which takes the values $0$ and $1$ with probability $0.5$, independently of the attributes. The goal of this simulation is to show that the influence of the features in the classification of the instances with response $b=1$ does not depend on the features' values. Table~\ref{tab:simulation2} and Figure~\ref{fig:simulation2} present the results obtained for this simulation. The computational time was $12.4$ minutes.

\begin{table}[H]
	\begin{center}
		\resizebox{7.5cm}{!}{
		\begin{tabular}{|c|c|c|c|}
			\hline
			$X_{j}, j \in K$ & $a_{j}$ & $\sum_{l\in K}I^{\Phi}_{l}(a_{j},b,K)$ & $I^{\Phi}_{j}(a_{j},b,K)$ \\ 
			\hline \hline
			\multirow{2}{*}{$X_{1}$} & 0 & -0.001 & -0.009\\
			\cline{2-4}
			& 1 & 0.017  & 0.011 \\
			\hline \hline
			\multirow{2}{*}{$X_{2}$} & 0 & -0.012 &  -0.019 \\
			\cline{2-4}
			& 1 & 0.026 & 0.023  \\ 
			\hline \hline
			\multirow{2}{*}{$X_{3}$} & 0 & 0.007  & -0.002 \\
			\cline{2-4}
			& 1 & 0.010 & 0.006\\
			\hline \hline
			\multirow{2}{*}{$X_{4}$} & 0 & 0.002 & -0.003 \\  
			\cline{2-4}
			& 1 & 0.014 & 0.005 \\
			\hline
		\end{tabular}
	}
		\caption{Results for simulation 2.}
		\label{tab:simulation2}
	\end{center}
\end{table}

Again, the outcomes are as expected: for each feature, there are barely differences in the values $I_{j}^{\Phi}(a_{j},b,K)$ and $\sum_{l\in K}I^{\Phi}_{l}(a_{j},b,K)$ when $a_{j}$ changes.
In this case, Datta et al.'s measure resulted in $(0.375,0.375,0.375,0.375)$. The response is not influenced by any one attribute more than the others. However, because the class was generated independently of the features, one would expect their influence to be zero.

\noindent
\begin{figure}[H]
\begin{minipage}[c]{6.3cm}
		\begin{center}
			\includegraphics[width=5.2cm]{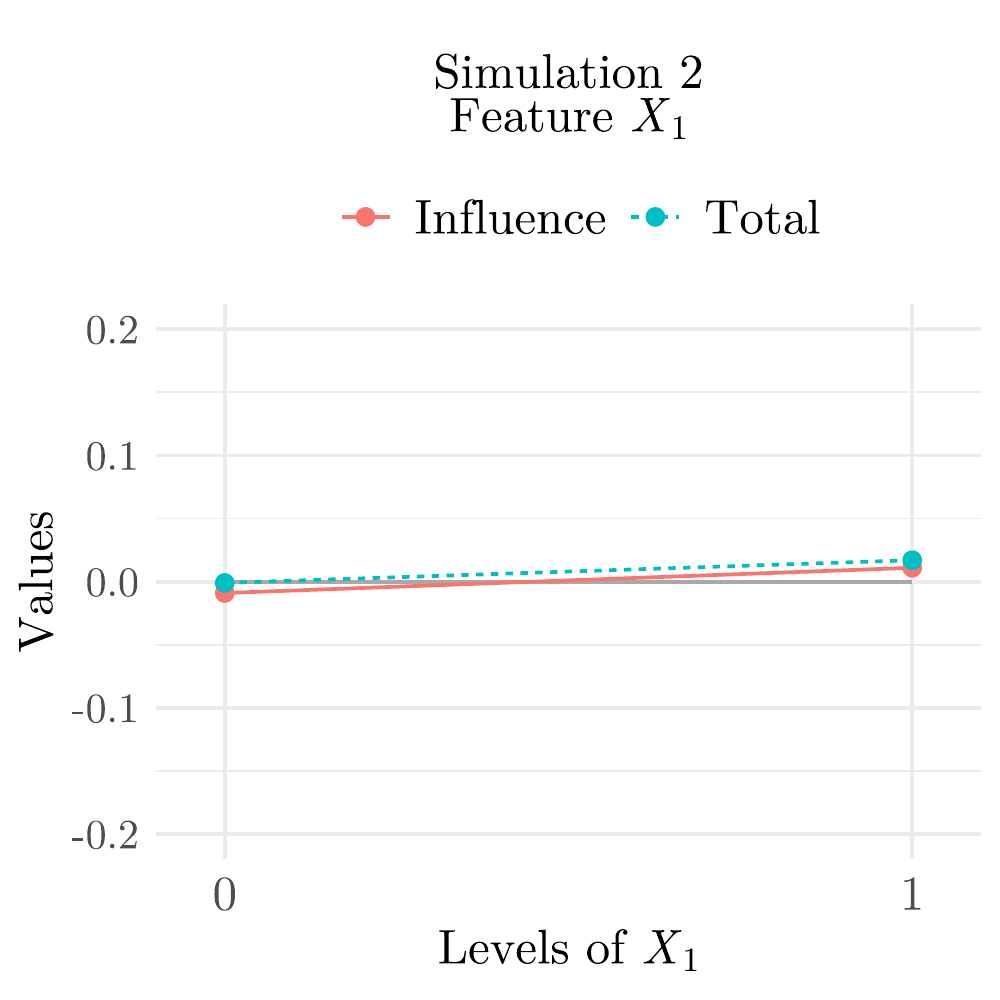}
		\end{center}
\end{minipage}
\hspace{0.2cm}
\begin{minipage}[c]{6.3cm}
		\begin{center}
			\includegraphics[width=5.2cm]{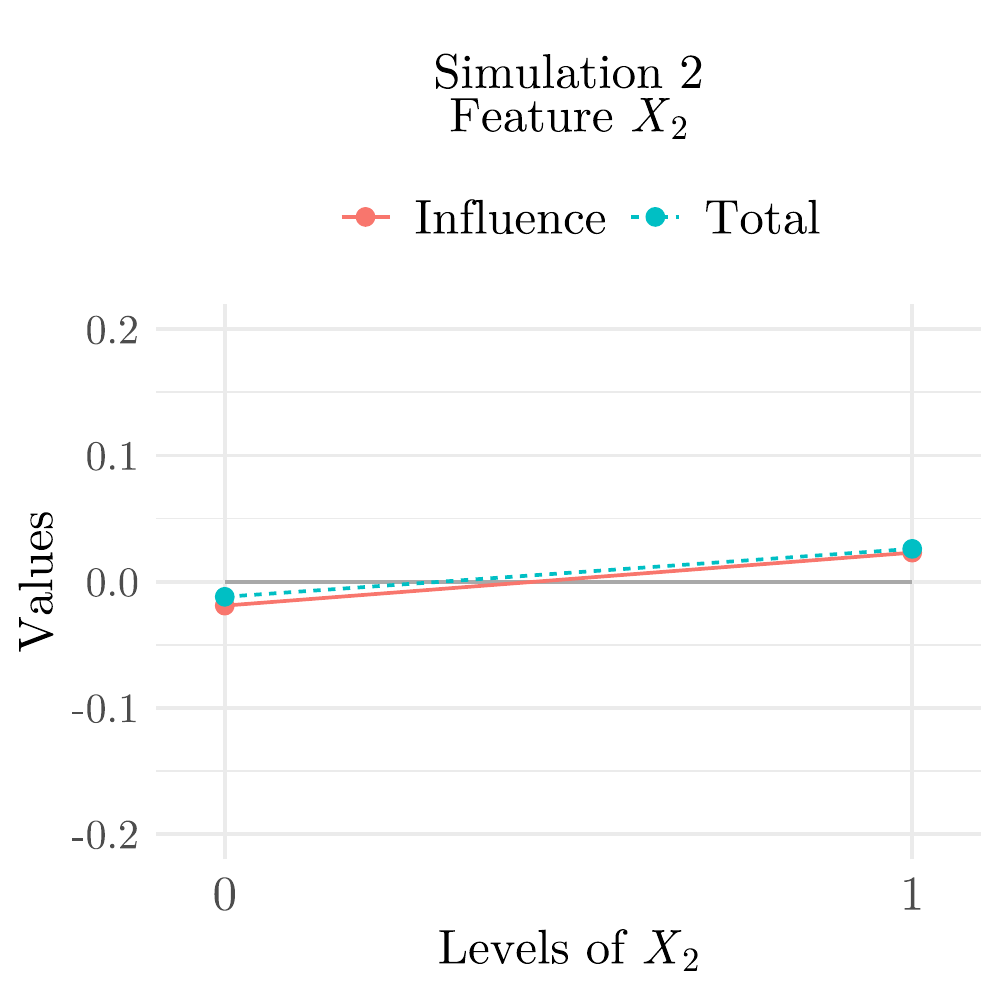}
		\end{center}
\end{minipage}

\noindent
\begin{minipage}[c]{6.3cm}
		\begin{center}
			\includegraphics[width=5.2cm]{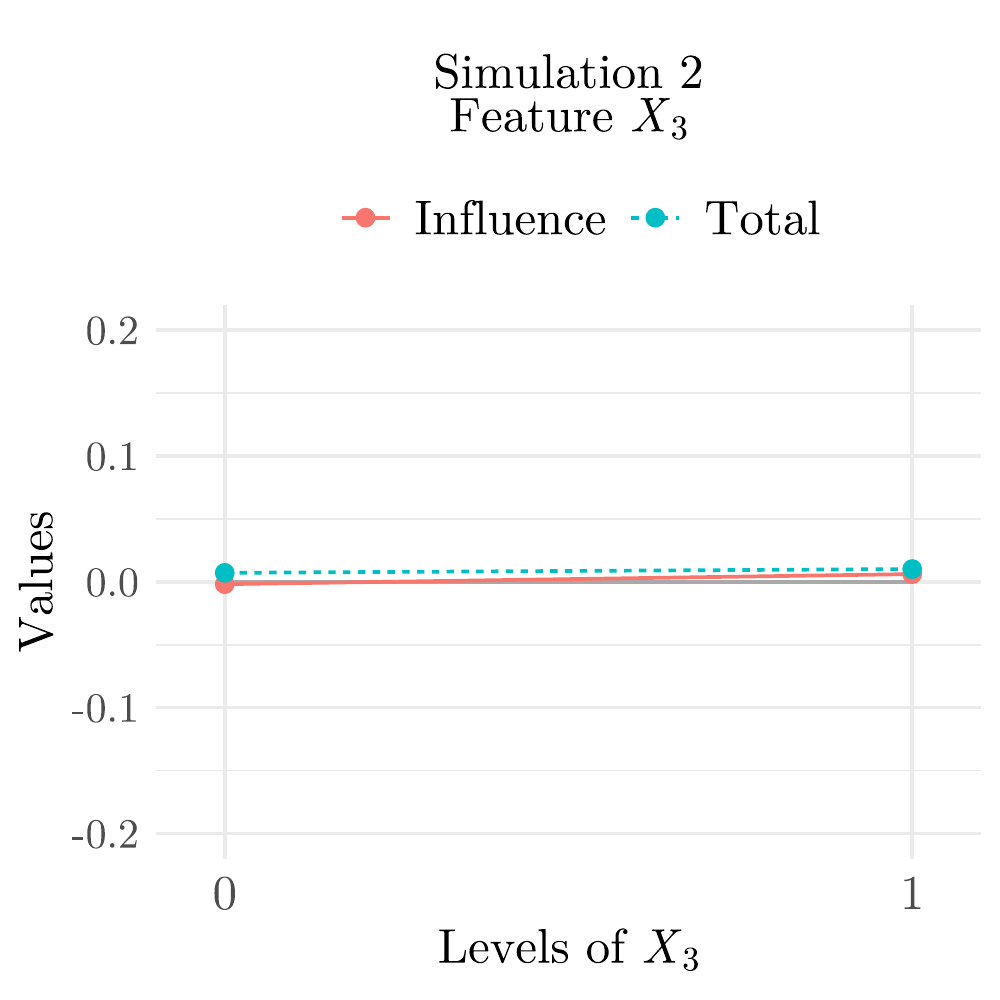}
		\end{center}
\end{minipage}
\hspace{0.2cm}
\begin{minipage}[c]{6.3cm}	
		\begin{center}
			\includegraphics[width=5.2cm]{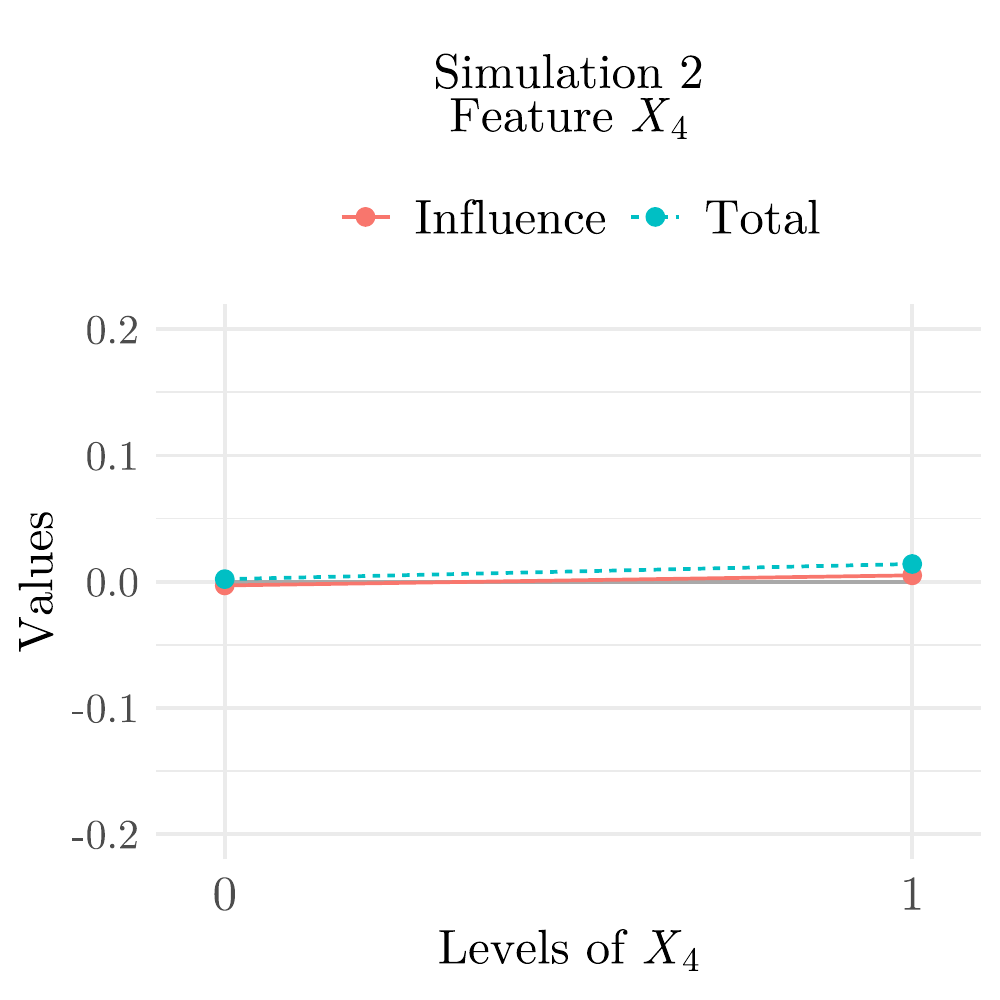}
		\end{center}
\end{minipage}
\caption{Influence and total influence for the features (Simulation 2).}
\label{fig:simulation2}
\end{figure}


Finally, we have considered the non-binary case. Now, the four attributes can take the values 0, 1 and 2 with equal probability, and the class of the response is computed as follows: in $1/3$ of the instances, it is the value of attribute $X_{1}$ that determines the response; while in the remaining $2/3$, it is attribute $X_{2}$ that determines it. Table~\ref{tab:simulation3} and 
Figure~\ref{fig:simulation3} illustrate the results. This took a runtime of $13.3$ minutes.

\begin{table}[H]
	\begin{center}
		\resizebox{7.3cm}{!}{
		\begin{tabular}{|c|c|c|c|}
			\hline
			$X_{j}, j \in K$ & $a_{j}$ & $\sum_{l\in K}I^{\Phi}_{l}(a_{j},b,K)$ & $I^{\Phi}_{j}(a_{j},b,K)$ \\ 
			\hline \hline
			\multirow{3}{*}{$X_{1}$} & 0 & 0.364 & -0.091 \\
			\cline{2-4}
			& 1 & 0.455 & 0.233 \\
			\cline{2-4}
			& 2 &  0.360 & -0.120   \\
			\hline \hline
			\multirow{3}{*}{$X_{2}$} & 0 &  0.105 & -0.172 \\
			\cline{2-4}
			& 1 & 0.495  & 0.445  \\ 
			\cline{2-4}
			& 2  & 0.005 & -0.245  \\ 
			\hline \hline
			\multirow{3}{*}{$X_{3}$} & 0  & 0.421 & -0.012  \\
			\cline{2-4}
			& 1  & 0.391 & 0.012 \\
			\cline{2-4}
			& 2  & 0.424  & 0.016\\
			\hline \hline
			\multirow{3}{*}{$X_{4}$} & 0 & 0.410 & 0.042 \\  
			\cline{2-4}
			& 1 & 0.385 & -0.041  \\
			\cline{2-4}
			& 2  & 0.435 & 0.020 \\
			\hline
		\end{tabular}
	}
		\caption{Results for simulation 3.}
		\label{tab:simulation3}
	\end{center}
\end{table}\vspace{-0.5cm}
The outcomes obtained show that changes in features $X_{3}$ and $X_{4}$ do not affect to the response being $b=1$, and their influence is almost zero whatever their values. Nevertheless, the value 1 of attributes $X_{1}$ and $X_{2}$ has a positive influence, which is larger in the case of the latter. On the contrary, when these attributes take the values $0$ and $2$, their influence is negative. This speaks against the class resulting in 1. In this case, the influence measure of Datta et al. is $(0.321, 1.827, 0.296, 0.296)$. This result shows that $X_{2}$ is the most influential feature, and that $X_{1}$ is more relevant than $X_{3}$ and $X_{4}$. Nevertheless, this measure does not properly capture the magnitude of how much more influential attribute $X_{1}$ is in comparison to $X_{3}$ and $X_{4}$.

\noindent
\begin{figure}[H]
\begin{minipage}[c]{6.3cm}
		\begin{center}
			\includegraphics[width=5cm]{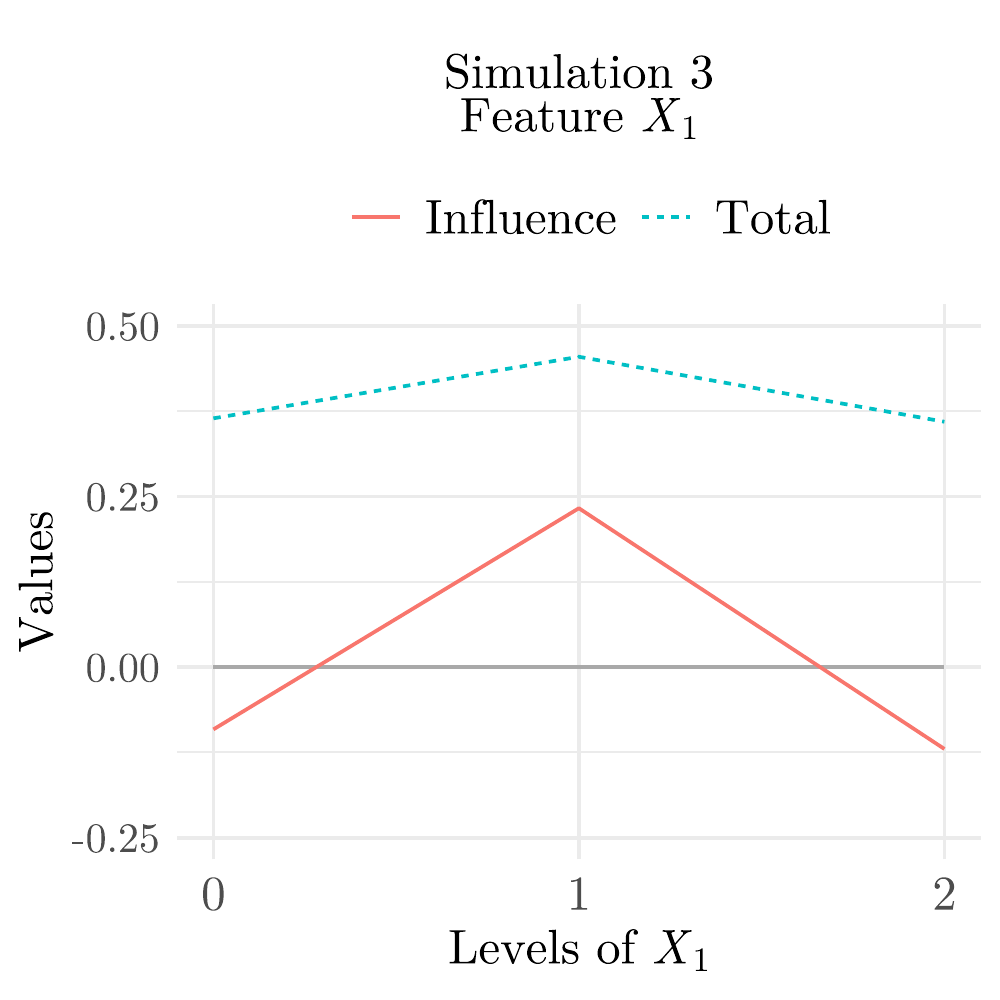}
		\end{center}
\end{minipage}
\hspace{0.1cm}
\begin{minipage}[c]{6.3cm}
		\begin{center}
			\includegraphics[width=5cm]{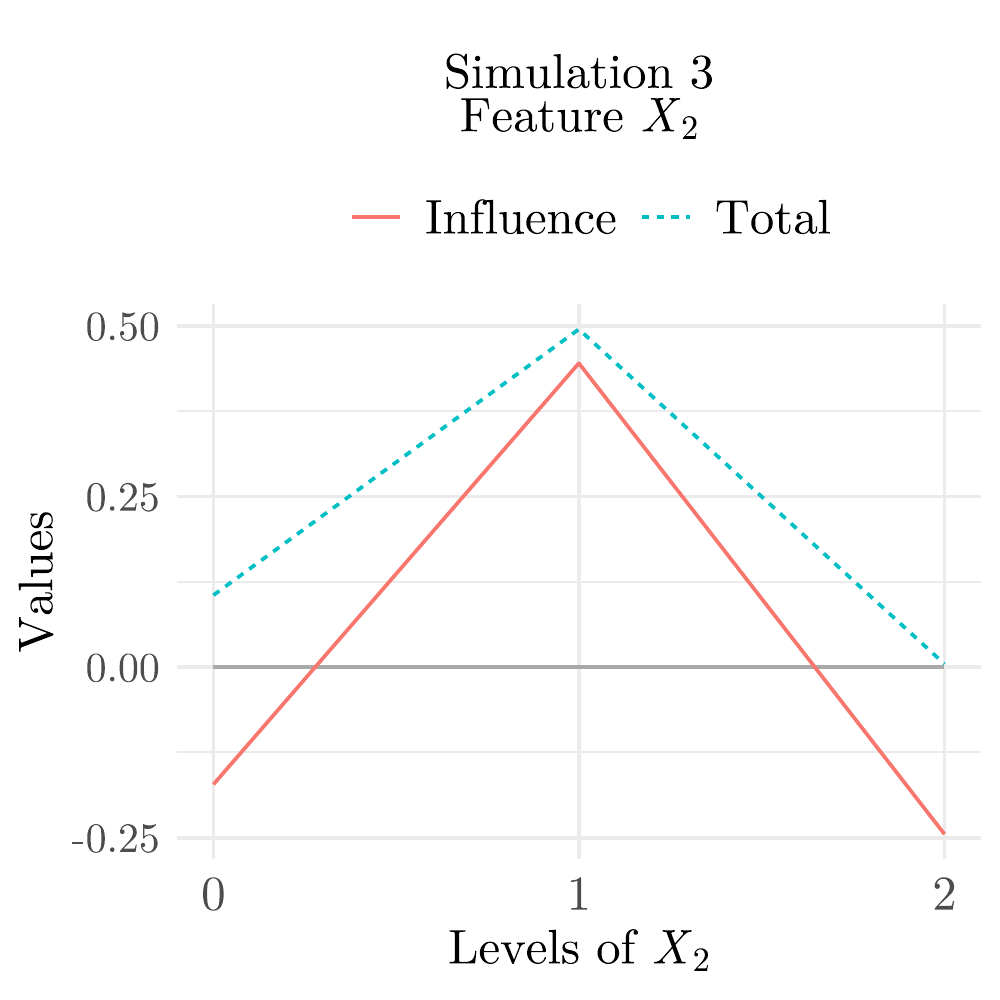}
		\end{center}
\end{minipage}

\noindent
\begin{minipage}[c]{6.3cm}
		\begin{center}
			\includegraphics[width=5cm]{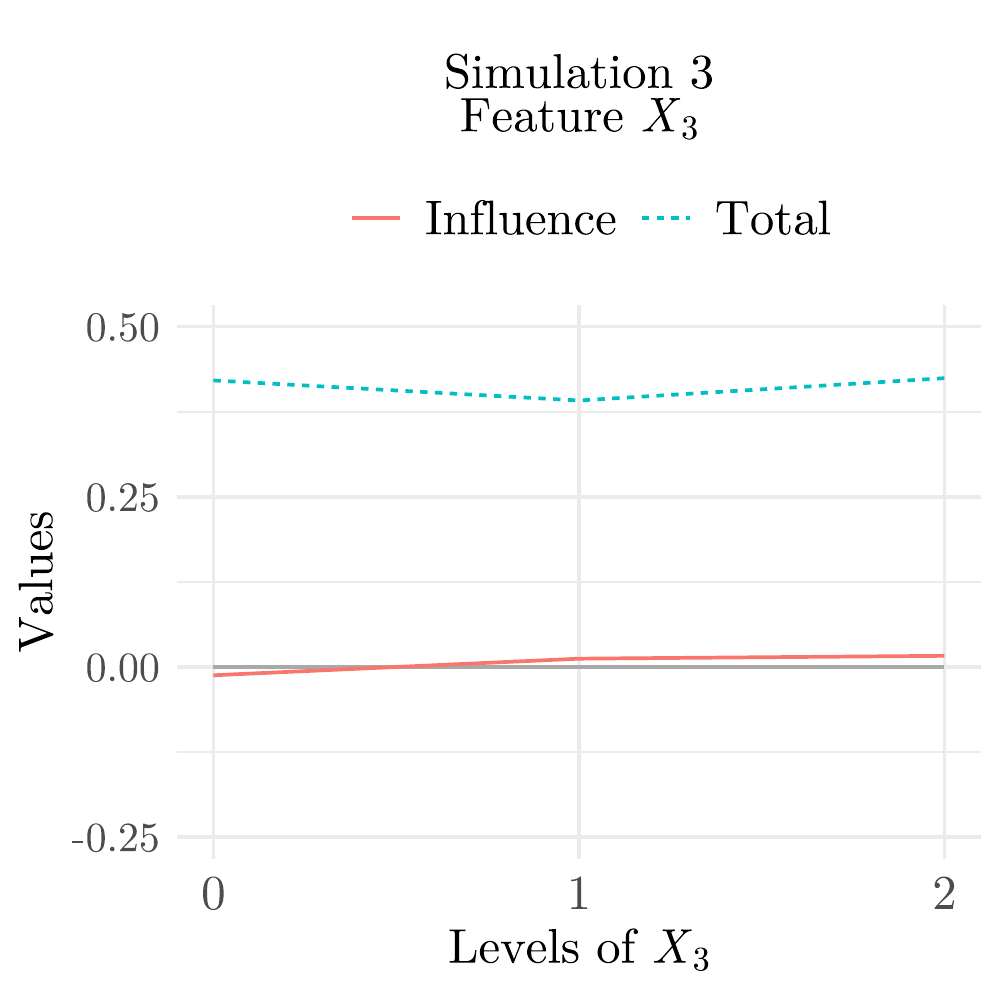}
		\end{center}
\end{minipage}
\hspace{0.2cm}
\begin{minipage}[c]{6.3cm}
		\begin{center}
			\includegraphics[width=5cm]{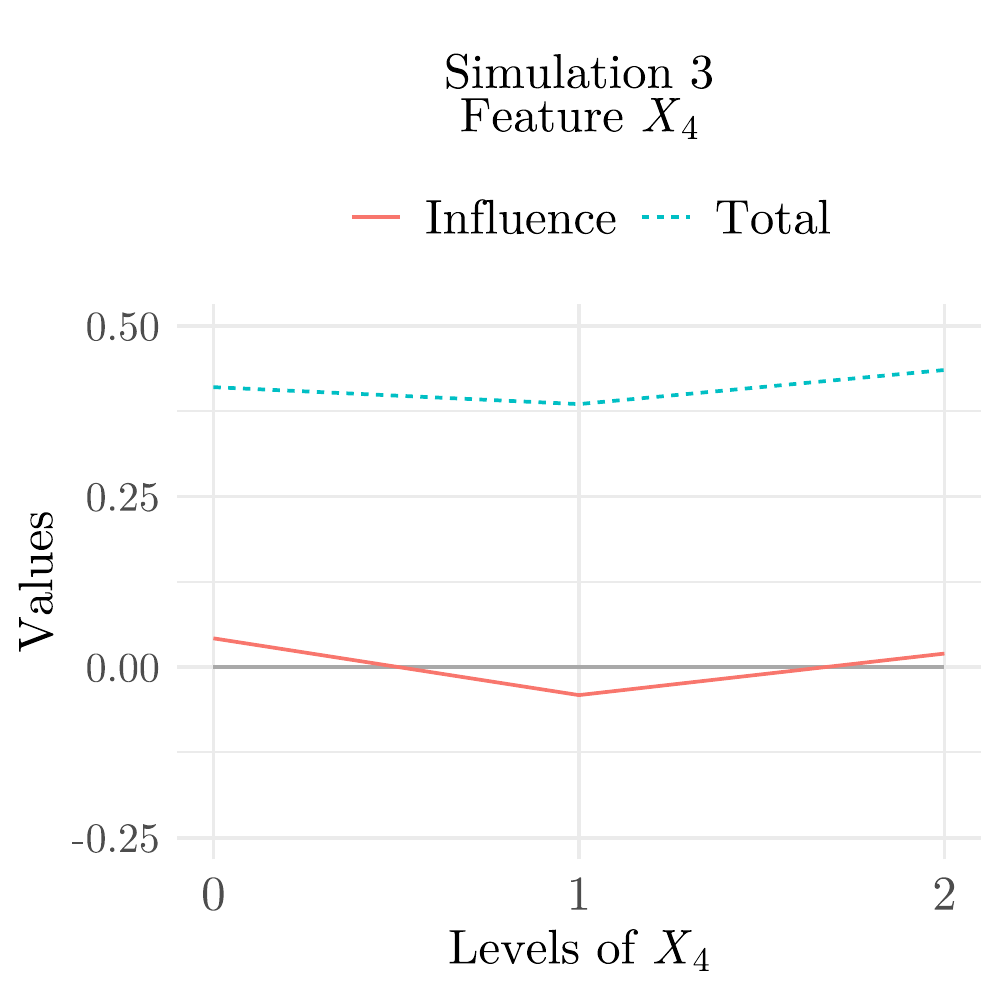}
		\end{center}
\end{minipage}
\caption{Influence and total influence for the features (Simulation 3).}
\label{fig:simulation3}
\end{figure}

In view of the previous results, our methodology seems to be appropriate to study the influence that the different feature values have on the classification of individuals. Since the experiments are satisfactory, this analytic tool can be applied to real-life problems. Consequently, this procedure has been employed on a real dataset concerning COVID-19 patients, whose results are presented in the next section. 

\section{Application of our influence measure to COVID-19 data}\label{sec:covid}

This section analyses a database of 10,454 patients from Galicia (a region in the northwest of Spain) infected with COVID-19 from March 6, 2020 to May 7, 2020. The objective is to study the influence of various patients' characteristics in three binary response variables of special interest: the need for hospitalisation, the need for ICU admission, and the eventual decease. The emphasis is not on the predictive classification of new patients, but on the analysis of the characteristics that influenced the patients whose complete history is known to have a positive response in the binary variables indicated.
On the other hand, what follows is not intended to be an exhaustive study of these data to draw definitive conclusions about the evolution of COVID-19, but simply an illustration of some of the uses of the measure of influence we introduced in Section~\ref{sec:theory}.

The features or attributes  which have been considered in this study are the following: 
\begin{itemize}
	 \setlength\itemsep{0pt}
	\item
	{\bf Sex}: 0 (woman), 1 (man).
	\item
	{\bf Age}: 0 (0-49 y/o), 1 (50-64 y/o), 2 (65-79 y/o), 3 (80 y/o and over).
	\item
	{\bf Cardiovascular diseases}: 0 (without diseases), 1 (mild diseases), 2 (severe diseases: ischaemia with angina, infarction, stroke).
	\item
	{\bf Respiratory diseases}: 0 (no diseases), 1 (mild diseases), 2 (severe diseases: malignancy, COPD, pneumonia).
	\item
	{\bf Metabolic diseases}: 0 (no diseases), 1 (mild diseases), 2 (severe diseases: malignancy, insulin-dependent diabetes).
	\item
	{\bf Urinary diseases}: 0 (none or mild diseases), 1 (severe diseases: malignancy, kidney failure).
\end{itemize}

The binary response variables considered in this application are:
\begin{itemize}
	 \setlength\itemsep{0pt}
	\item
	{\bf Decease (exitus)}: 0 (no), 1 (yes).
	\item
	{\bf ICU admission}: 0 (no), 1 (yes).
	\item
	{\bf Need for hospitalisation}: 0 (no), 1 (yes).
\end{itemize}

Next, we applied the methodology outlined in Section~\ref{sec:theory} to measure the influence of the features in the classification with respect to the binary response variables. For instance, the interest would reside in selecting those individuals who resulted in decease (that is, $\texttt{decease}=1$) when our purpose is to know the most influential attributes for the exitus. 
Note that to estimate the influence of feature $X_j$ on $Y$, we use the influence that $X_j$ has in the classification of the elements of the sample ${\cal M}$ using an excellent classifier, since it is precisely trained with the sample ${\cal M}$. As in the previous section, we use the random forest classifier introduced by \cite{Breiman2001} and implemented in \texttt{R} through the \texttt{RWeka} library.

 Let {\small $\{X_{1} = \texttt{sex}, X_{2} = \texttt{age}, X_{3} = \texttt{cardi}, X_{4}=\texttt{resp}, X_{5} = \texttt{meta}, X_{6} =~\texttt{uri}\}$} be the set of features. We start the analysis by presenting Figures~\ref{fig:decease}, \ref{fig:icu} and \ref{fig:hos}, which display the influence and total influence of the different features' values on the three classification problems.
 Let us explain in more detail what the graphics in the figures show. In each of the graphics a response variable is chosen and set its value to $1$, and also a feature is chosen. The graphic shows in red the measure of influence of the chosen feature when we set its value to each of the possible values it can take (feature influence), and in blue the sum of the measures of influence of all the features (total influence). The objective of these figures is to identify what we call {\em influence scenarios}. An influence scenario is detected when the total influence shown in the corresponding graphic deviates noticeably from zero. 
 
 \noindent
 \begin{figure}[H]
 	\begin{minipage}[c]{6.3cm}
 		\begin{center}
 			\includegraphics[width=5.5cm]{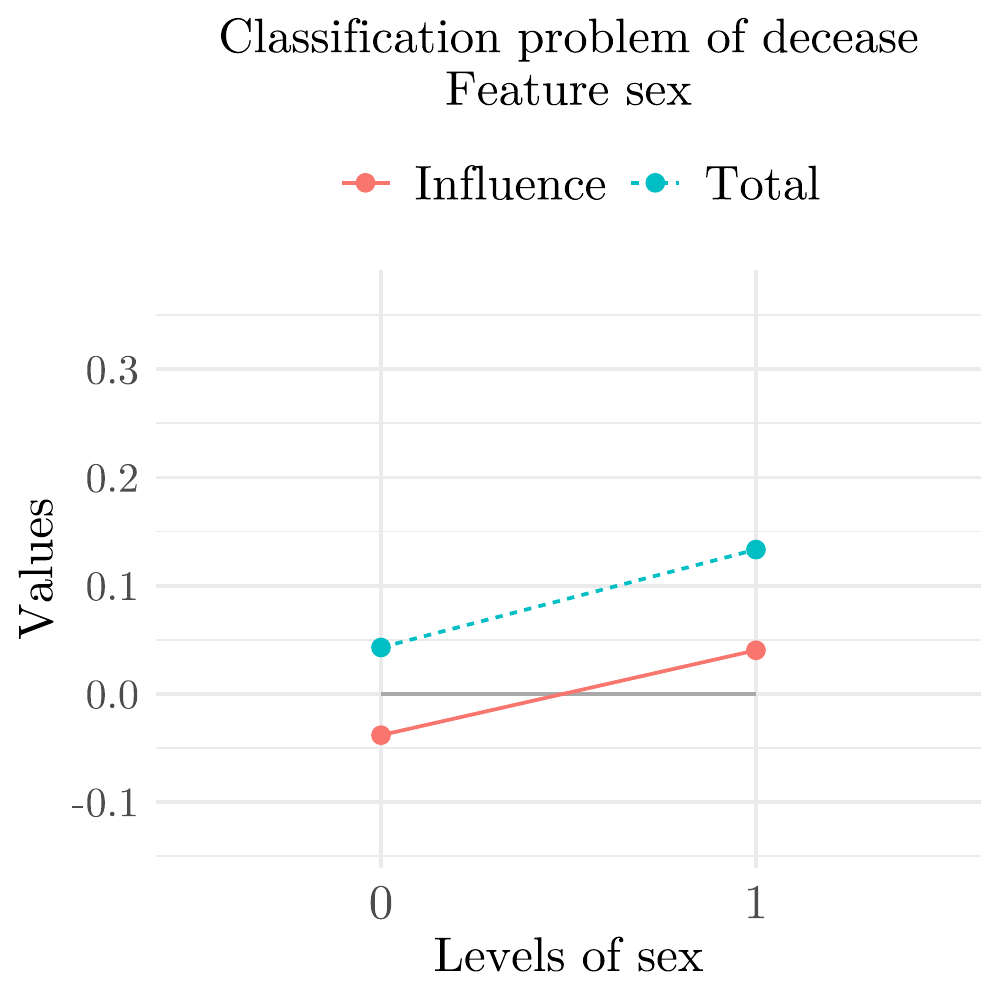}
 		\end{center}
 	\end{minipage}
 	\hspace{0.1cm}
 	\begin{minipage}[c]{6.3cm}
 		\begin{center}
 			\includegraphics[width=5.5cm]{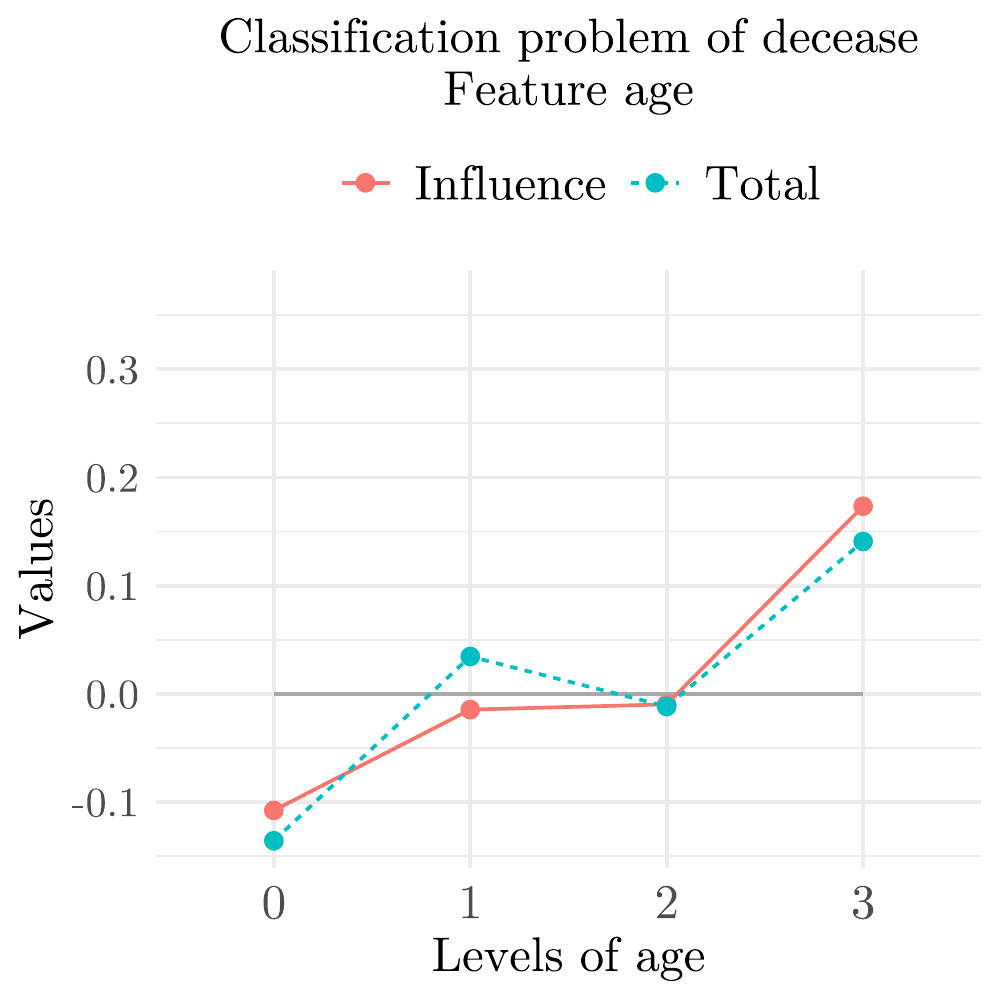}
 		\end{center}
 	\end{minipage}
 	
 	\noindent
 	\begin{minipage}[c]{6.3cm}
 		\begin{center}
 			\includegraphics[width=5.5cm]{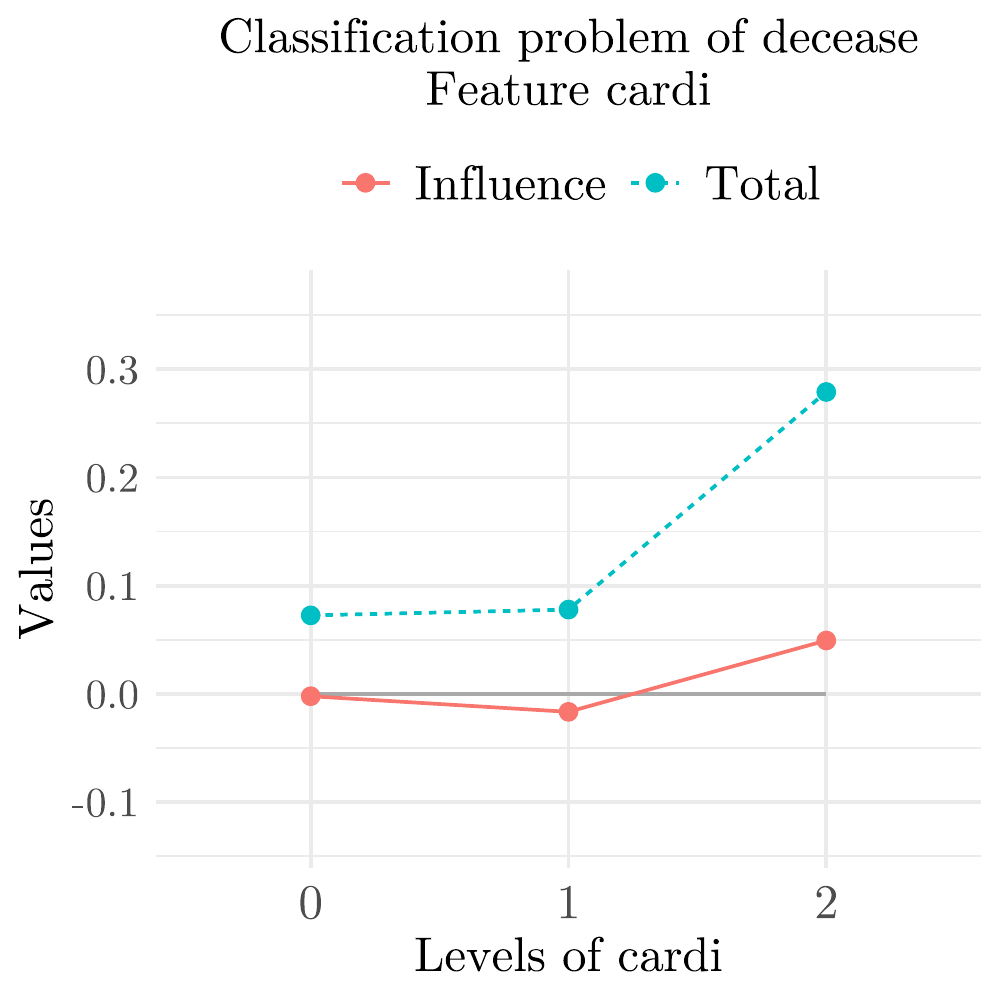}
 		\end{center}
 	\end{minipage}
 	\hspace{0.2cm}
 	\begin{minipage}[c]{6.3cm}
 		\begin{center}
 			\includegraphics[width=5.5cm]{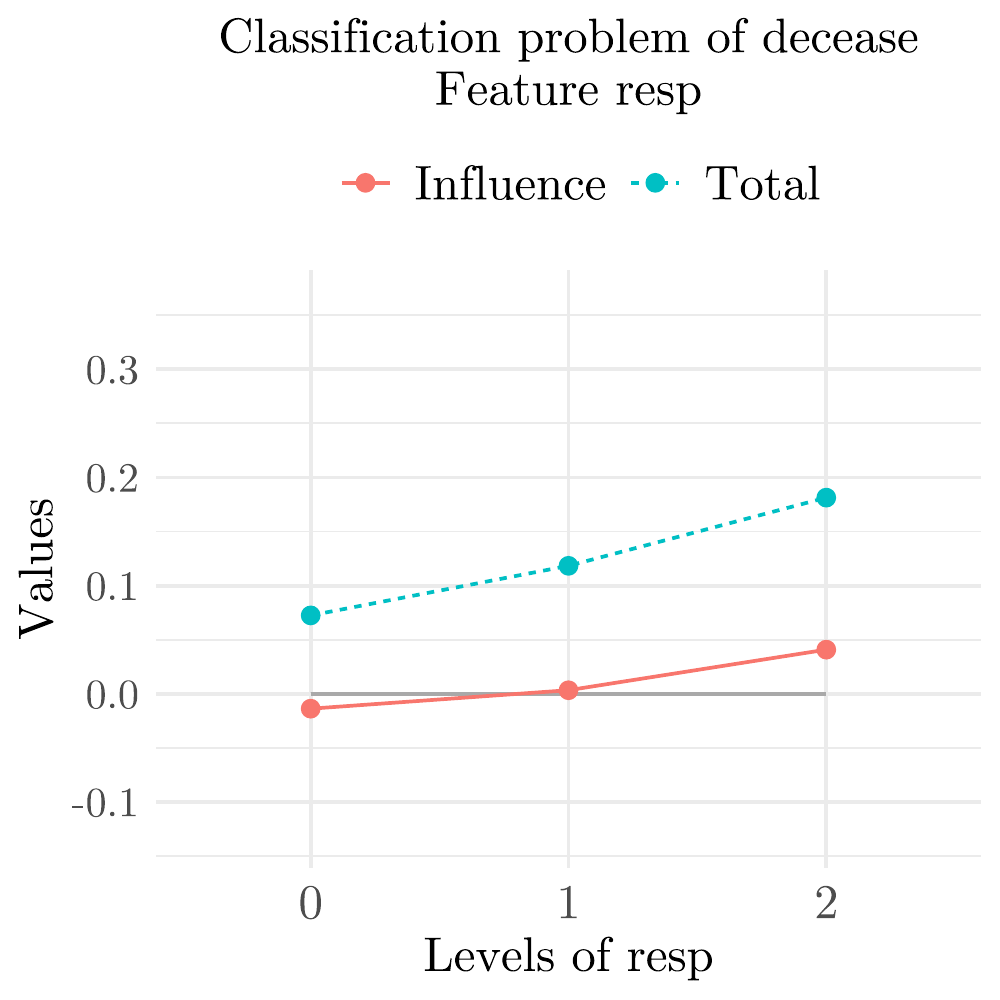}
 		\end{center}
 	\end{minipage}
 
 	\noindent
 	\begin{minipage}[c]{6.3cm}
 		\begin{center}
 			\includegraphics[width=5.5cm]{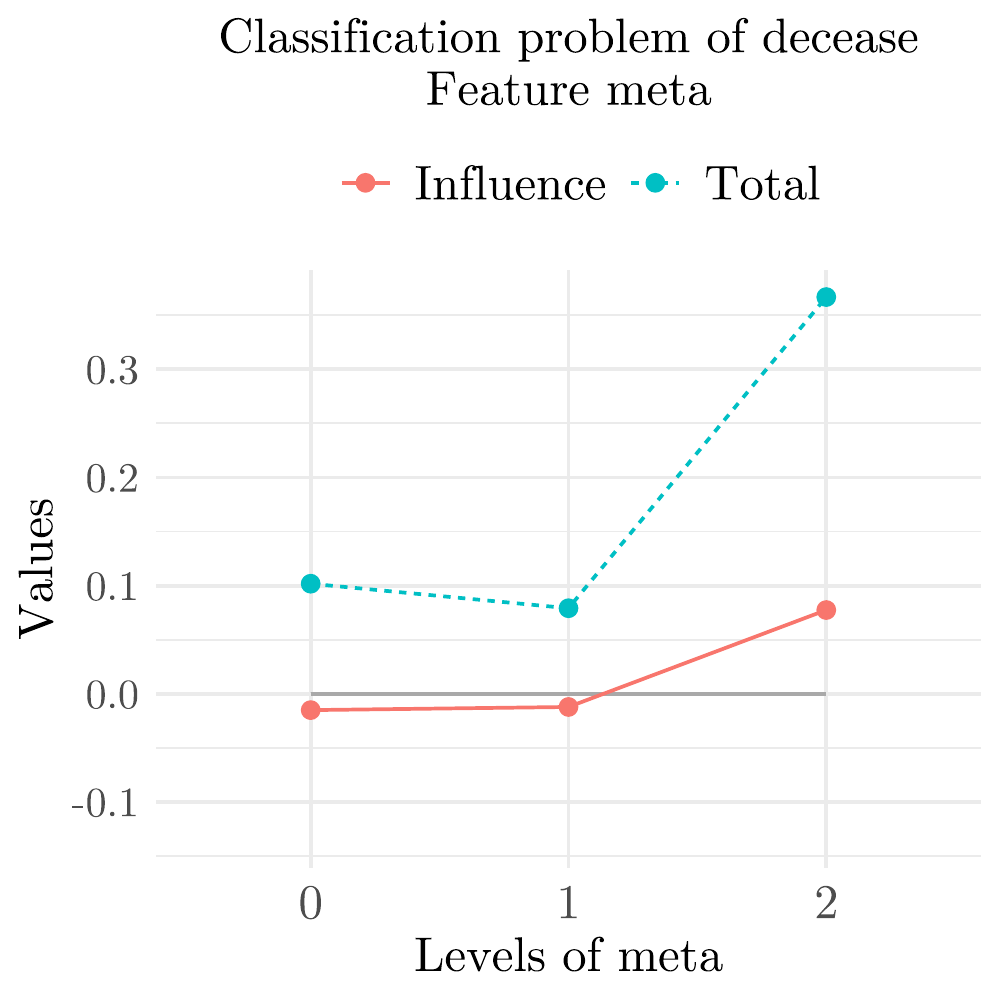}
 		\end{center}
 	\end{minipage}
 	\hspace{0.1cm}
 	\begin{minipage}[c]{6.3cm}
 		\begin{center}
 			\includegraphics[width=5.5cm]{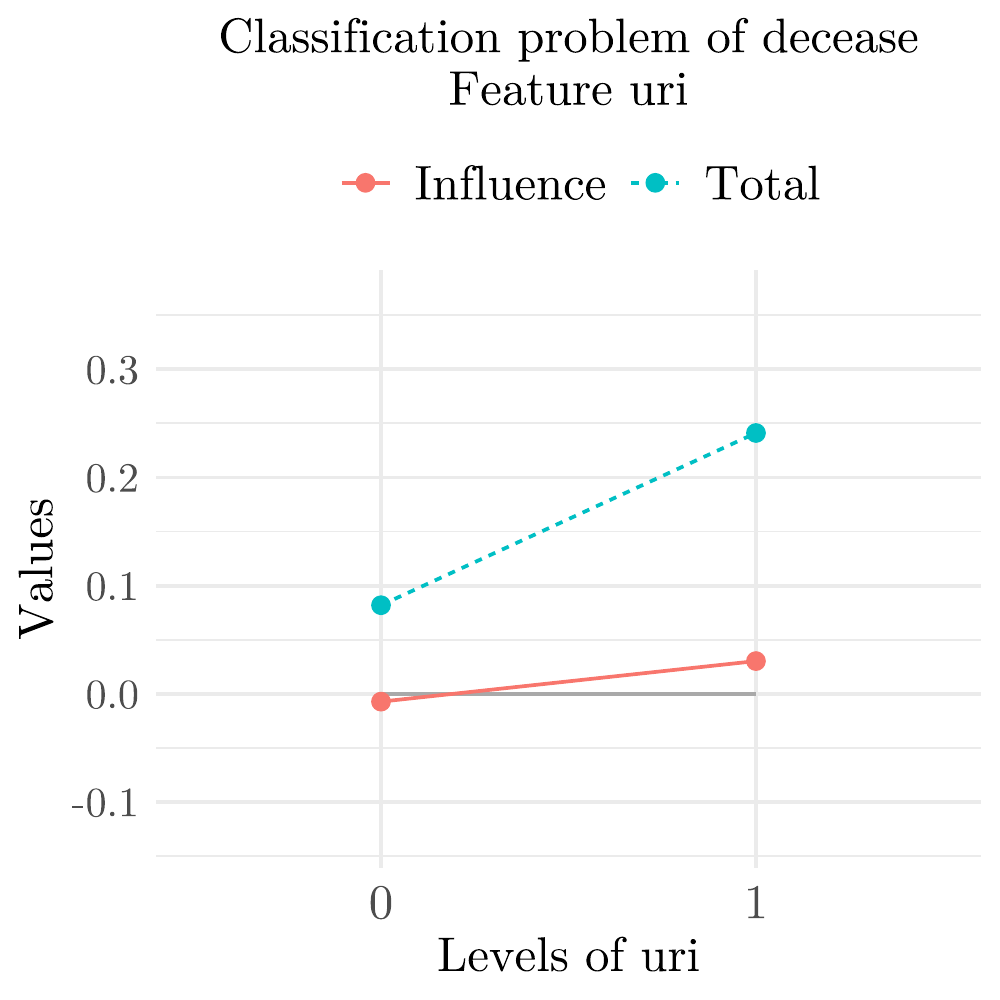}
 		\end{center}
 	\end{minipage}
 	\caption{Influence and total influence for the features on the decease.}
 	\label{fig:decease}
 \end{figure}

 \noindent
 \begin{figure}[H]
 	\begin{minipage}[c]{6.3cm}
 		\begin{center}
 			\includegraphics[width=5.5cm]{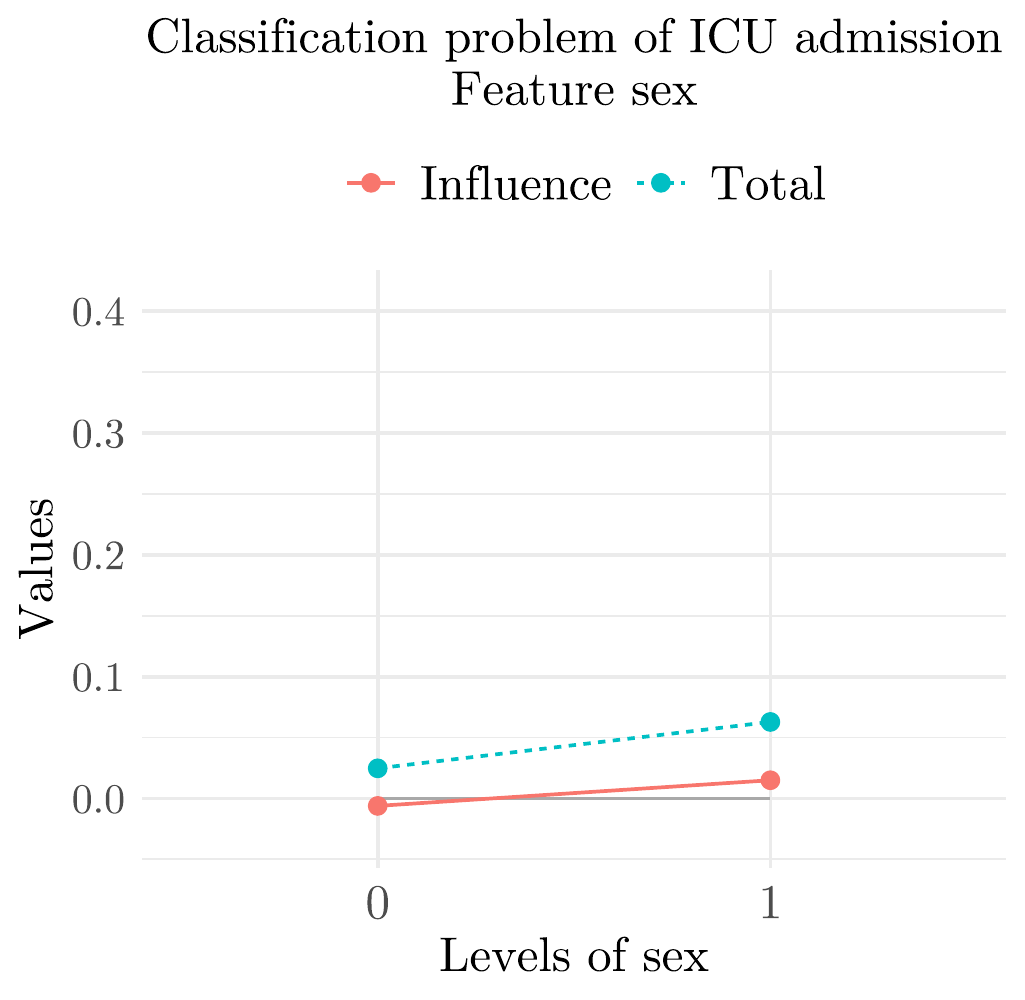}
 		\end{center}
 	\end{minipage}
 	\hspace{0.1cm}
 	\begin{minipage}[c]{6.3cm}
 		\begin{center}
 			\includegraphics[width=5.5cm]{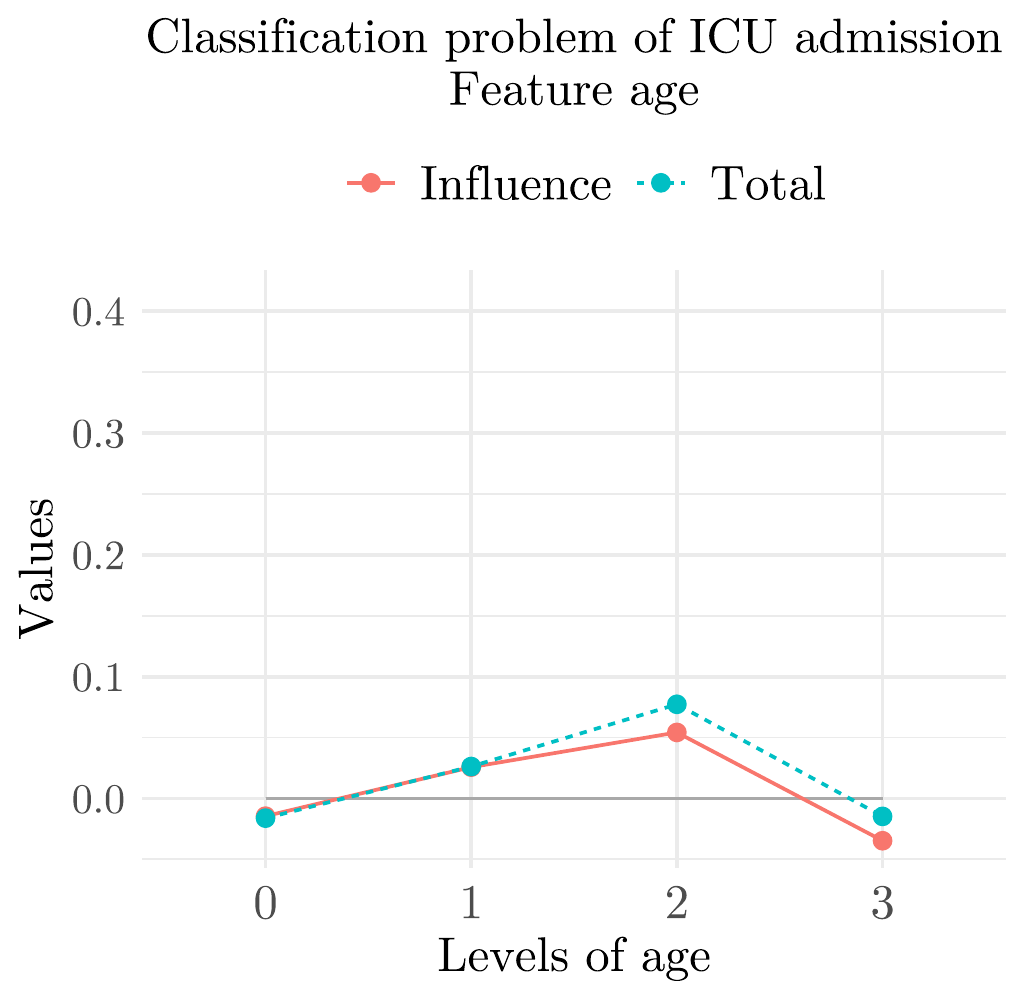}
 		\end{center}
 	\end{minipage}
 	
 	\noindent
 	\begin{minipage}[c]{6.3cm}
 		\begin{center}
 			\includegraphics[width=5.5cm]{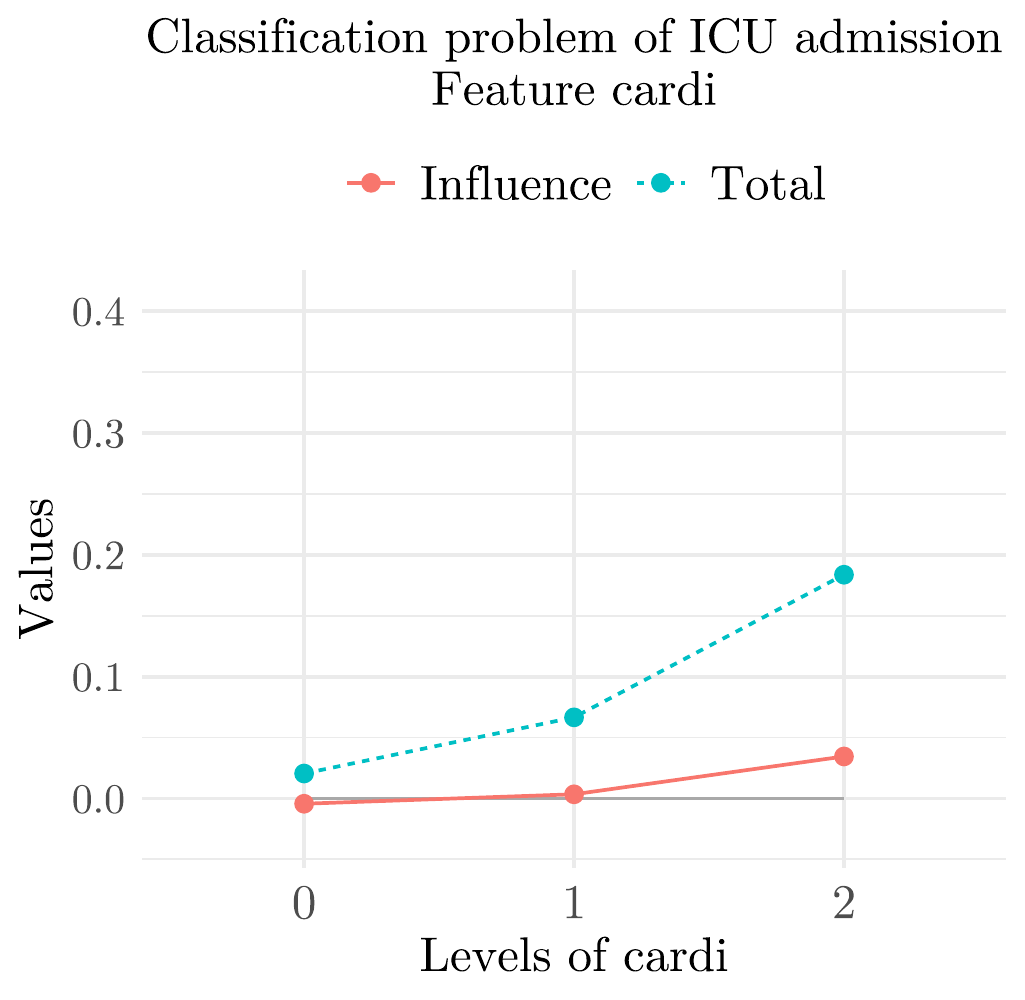}
 		\end{center}
 	\end{minipage}
 	\hspace{0.2cm}
 	\begin{minipage}[c]{6.3cm}
 		\begin{center}
 			\includegraphics[width=5.5cm]{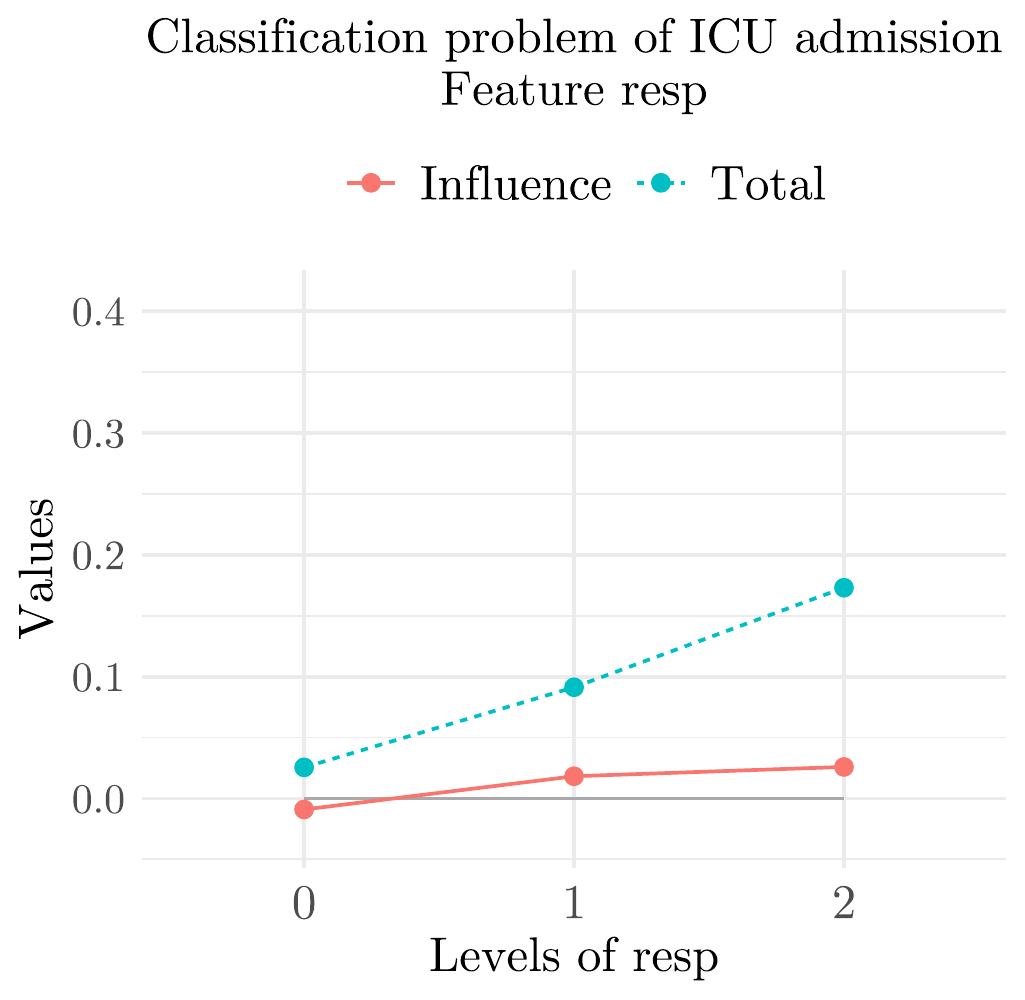}
 		\end{center}
 	\end{minipage}
 	
 	\noindent
 	\begin{minipage}[c]{6.3cm}
 		\begin{center}
 			\includegraphics[width=5.5cm]{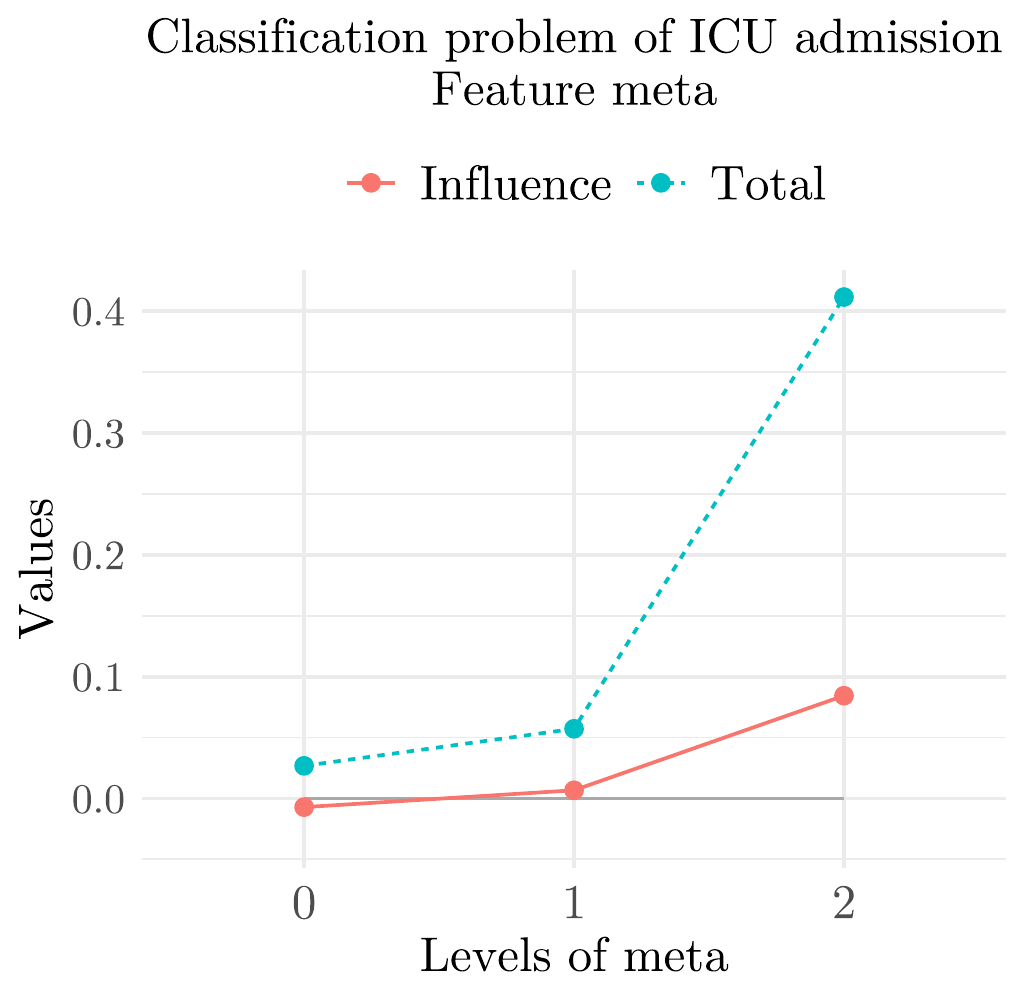}
 		\end{center}
 	\end{minipage}
 	\hspace{0.1cm}
 	\begin{minipage}[c]{6.3cm}
 		\begin{center}
 			\includegraphics[width=5.5cm]{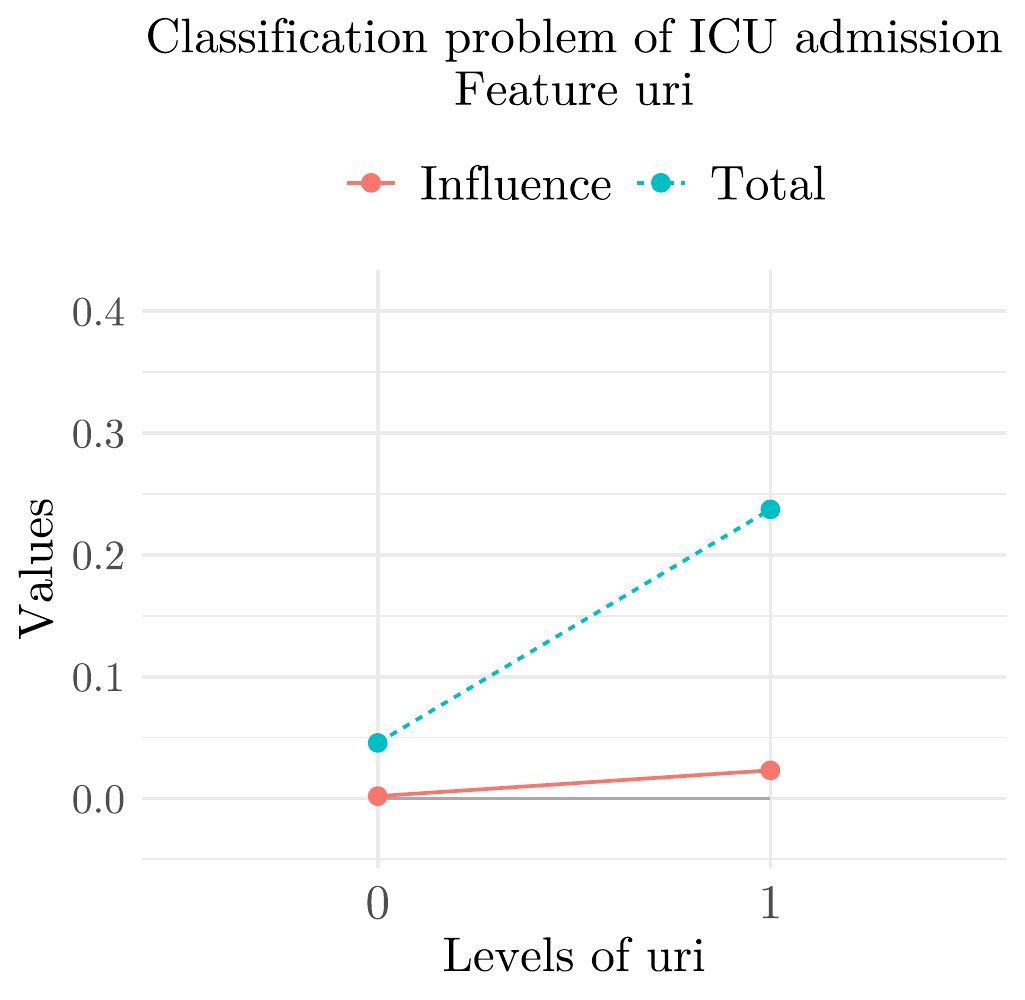}
 		\end{center}
 	\end{minipage}
 	\caption{Influence and total influence for the features on the ICU admission.}
 	\label{fig:icu}
 \end{figure}

 \noindent
 \begin{figure}[H]
 	\begin{minipage}[c]{6.3cm}
 		\begin{center}
 			\includegraphics[width=5.5cm]{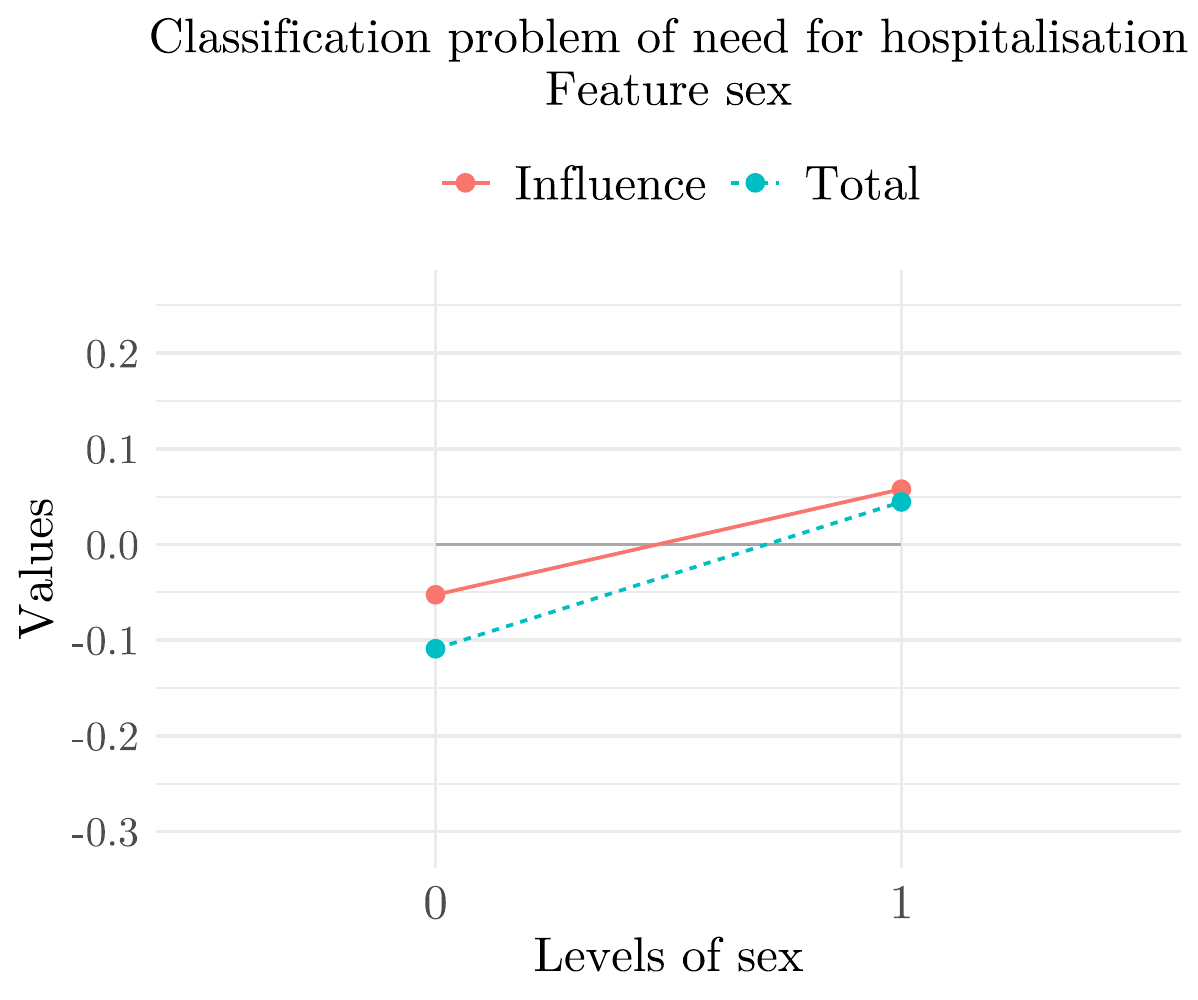}
 		\end{center}
 	\end{minipage}
 	\hspace{0.1cm}
 	\begin{minipage}[c]{6.3cm}
 		\begin{center}
 			\includegraphics[width=5.5cm]{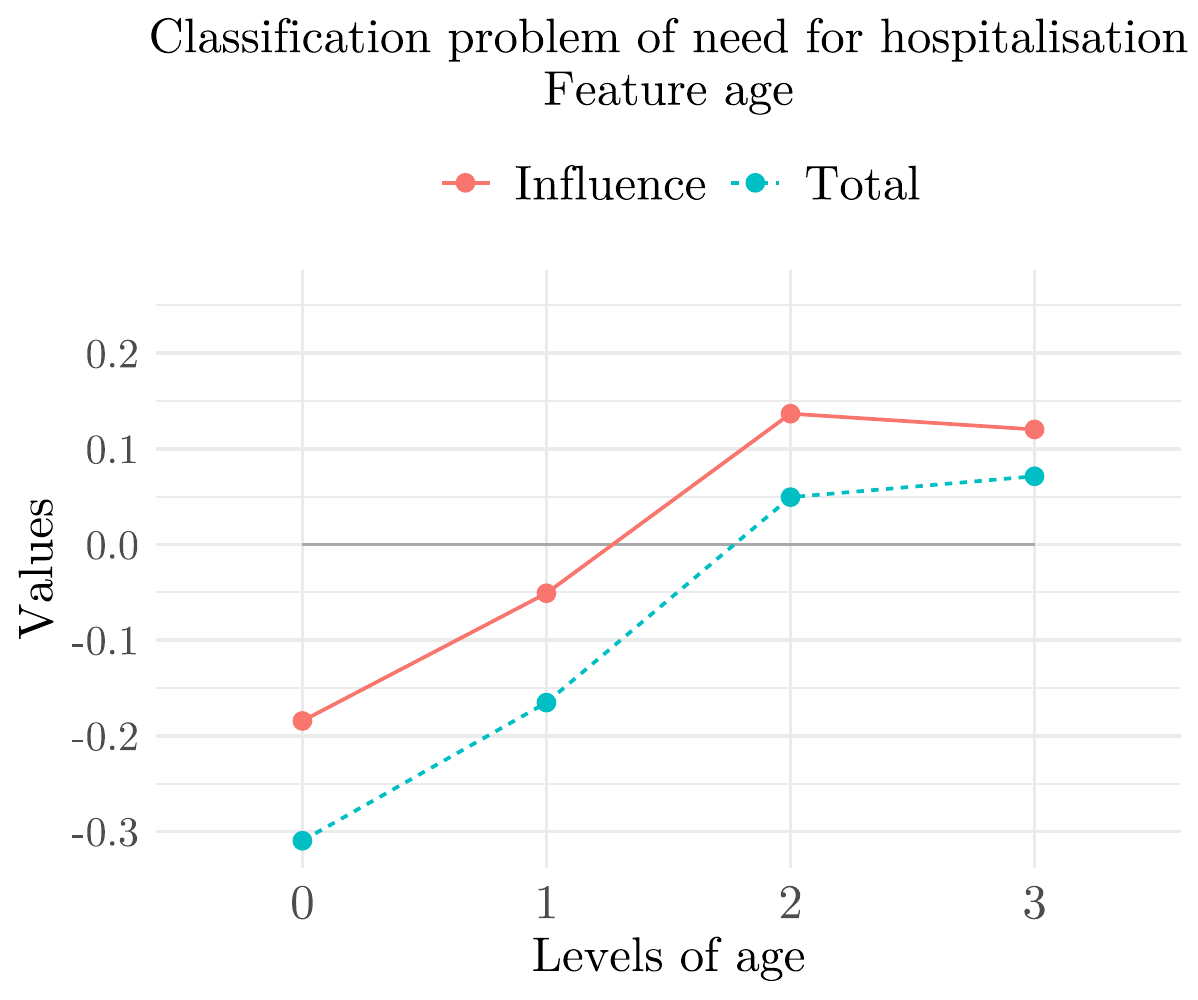}
 		\end{center}
 	\end{minipage}
 	
 	\noindent
 	\begin{minipage}[c]{6.3cm}
 		\begin{center}
 			\includegraphics[width=5.5cm]{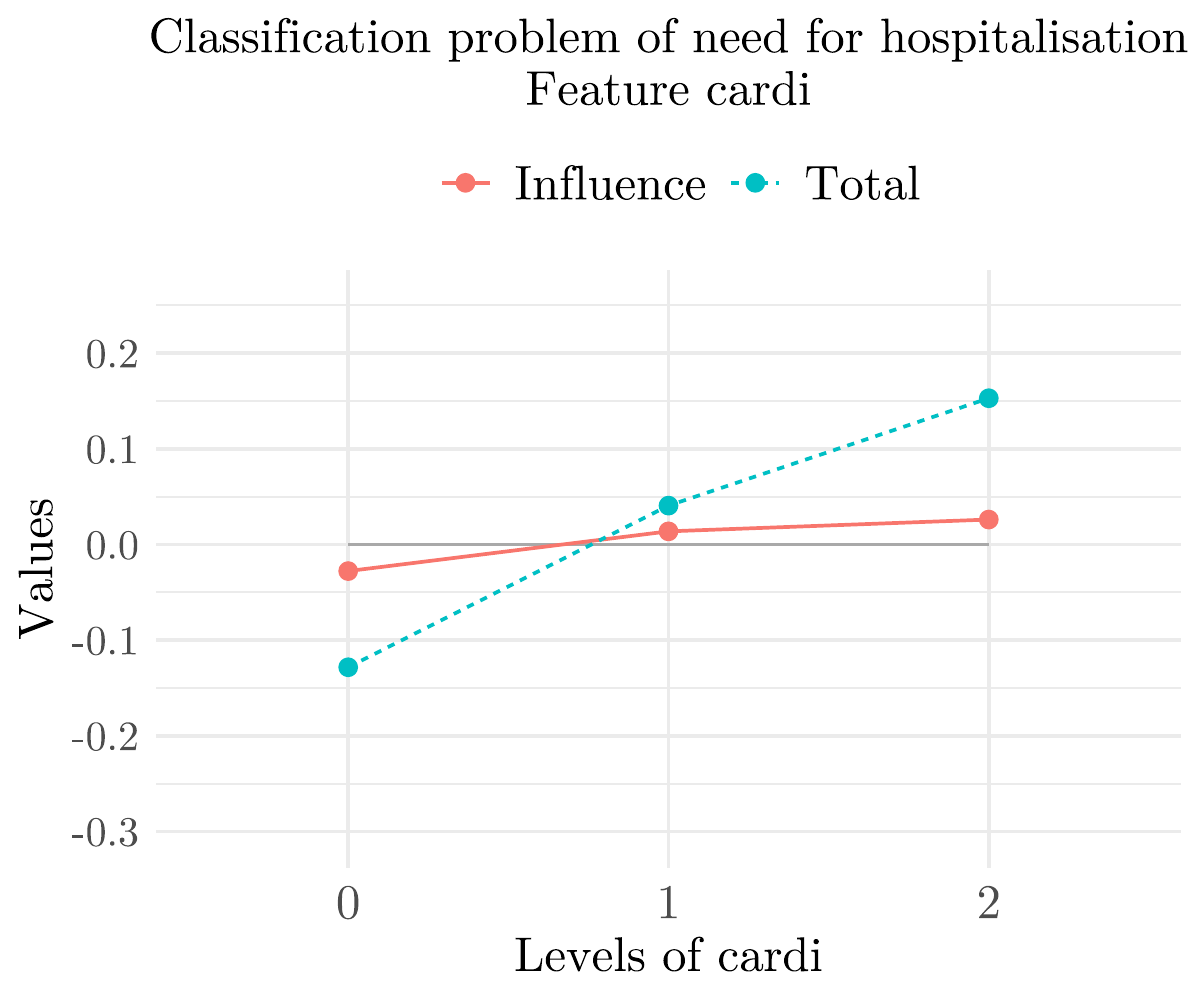}
 		\end{center}
 	\end{minipage}
 	\hspace{0.2cm}
 	\begin{minipage}[c]{6.3cm}
 		\begin{center}
 			\includegraphics[width=5.5cm]{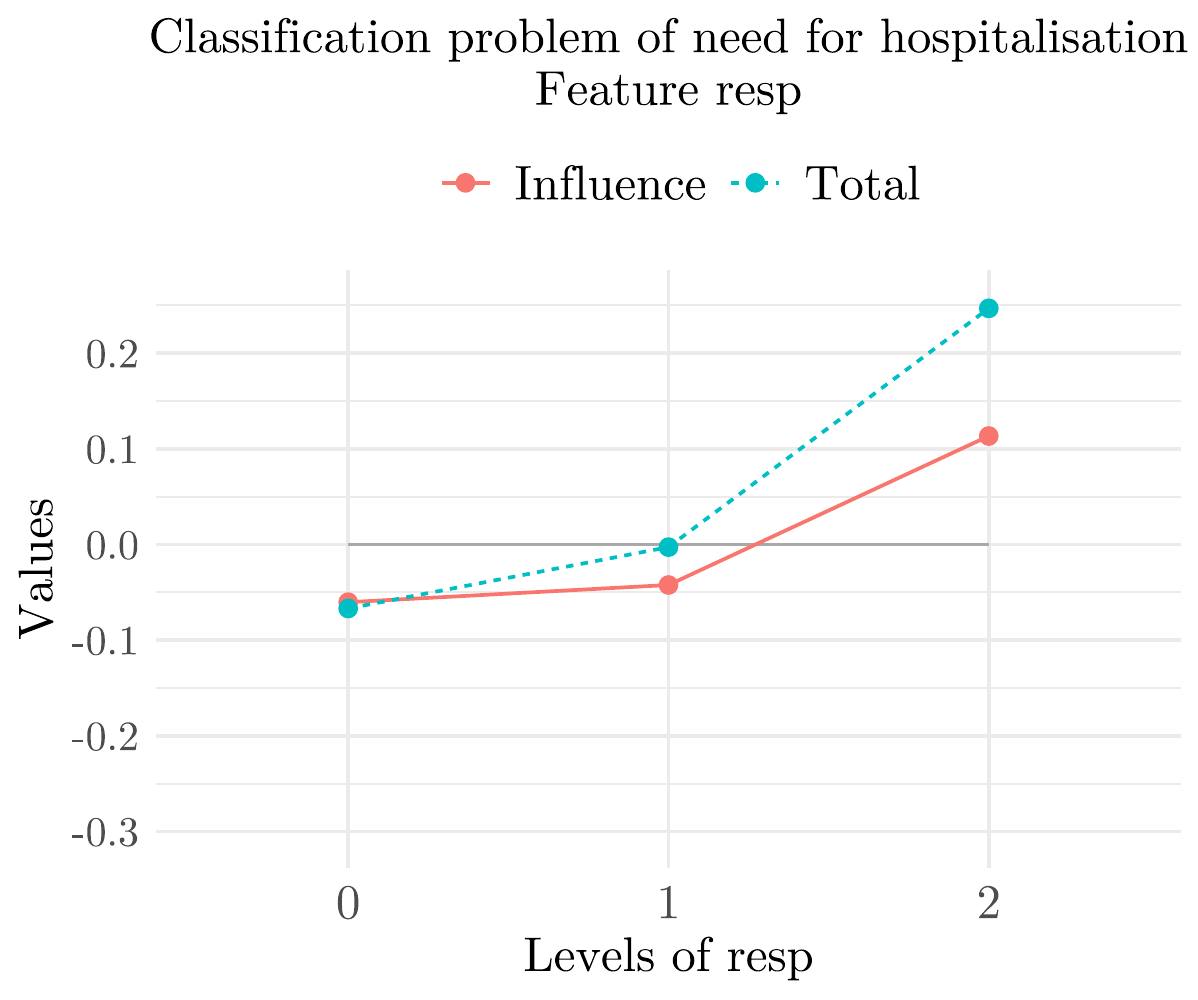}
 		\end{center}
 	\end{minipage}
 	
 	\noindent
 	\begin{minipage}[c]{6.3cm}
 		\begin{center}
 			\includegraphics[width=5.5cm]{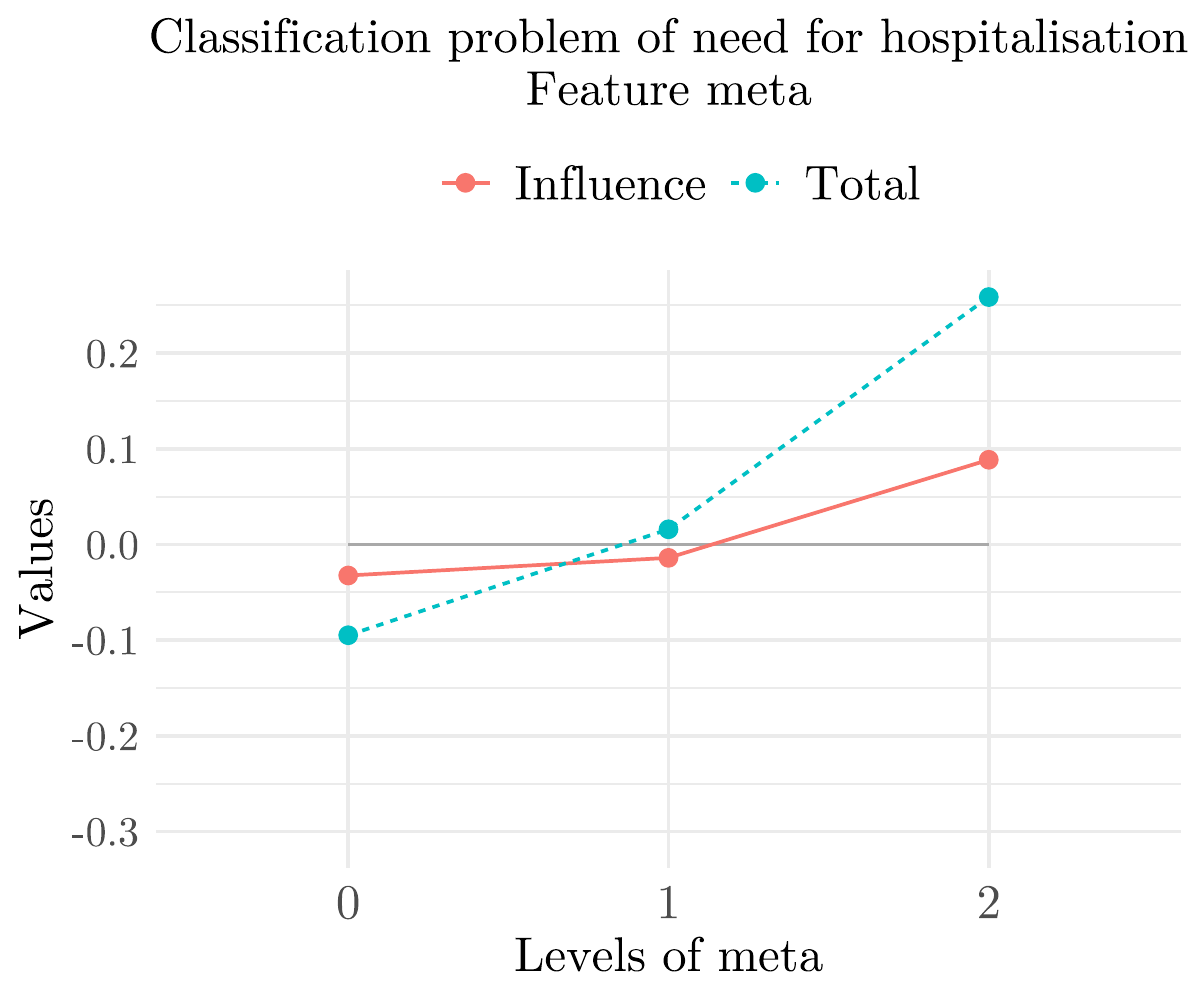}
 		\end{center}
 	\end{minipage}
 	\hspace{0.1cm}
 	\begin{minipage}[c]{6.3cm}
 		\begin{center}
 			\includegraphics[width=5.5cm]{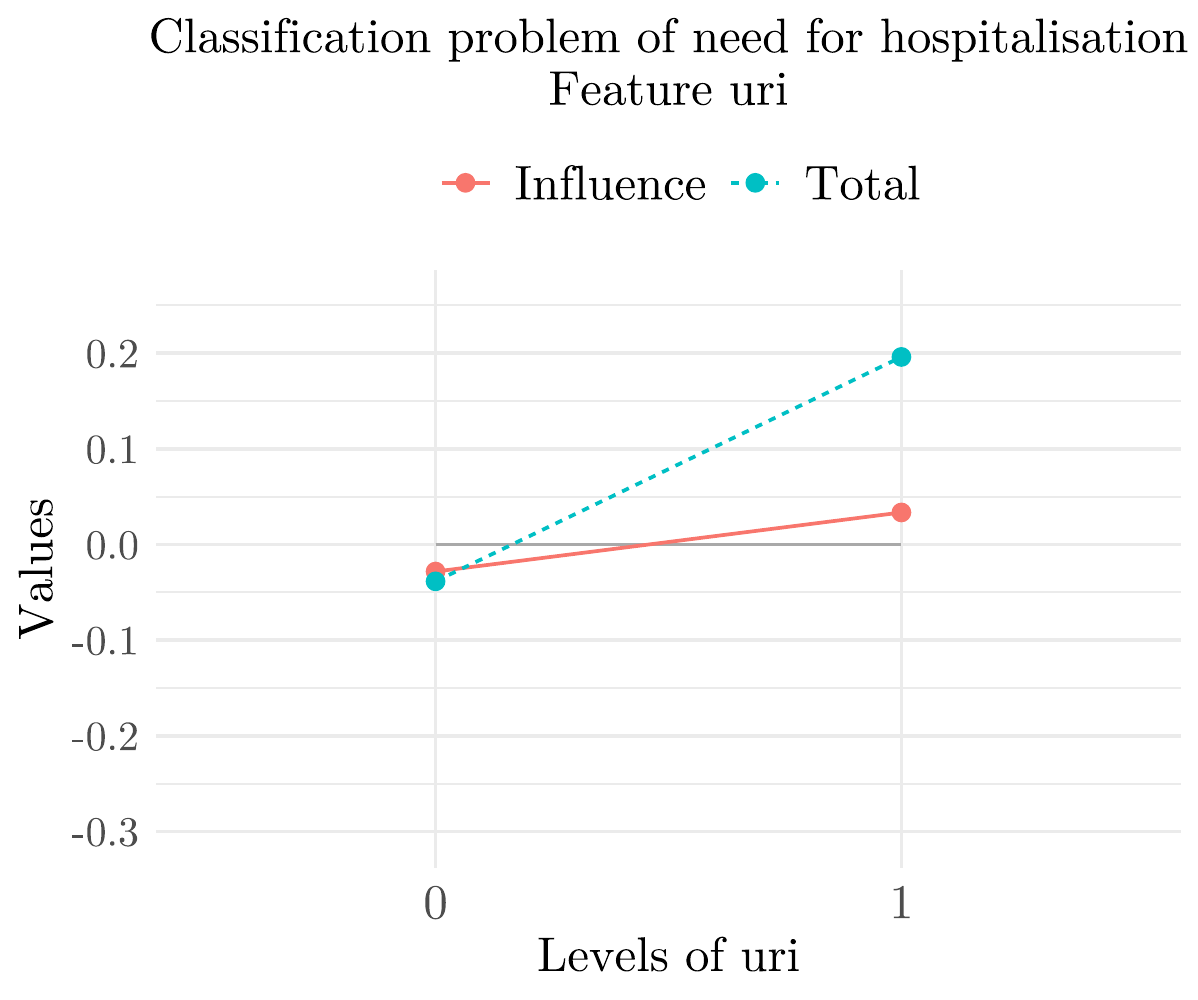}
 		\end{center}
 	\end{minipage}
 	\caption{Influence and total influence for the features on the need for hospitalisation.}
 	\label{fig:hos}
 \end{figure}

 For example, in Figure~\ref{fig:decease} several influence scenarios can be identified. The first is the case of \texttt{age}, both when it is worth $0$ and when it is worth $3$. There are two influence scenarios here that allow us to state that in the case of young individuals ($\texttt{age}=0$) and in the case of old individuals ($\texttt{age}=3$) we detect an important influence of the features on mortality, negative in the first case and positive in the second. We can observe that in this graphic the red and blue lines (\texttt{age} influence and total influence, respectively) are very close, which means that this total influence is mainly due to age.
 
 Other influence scenarios that can be inferred from the figure are those corresponding to the feature \texttt{cardi} being $2$ and the feature \texttt{meta} being $2$. Note, however, that in such scenarios the red and blue lines are noticeably separated, which means that the significant total influence detected is not primarily due to the features chosen in each case. Therefore, for each of these two scenarios, Table~\ref{tab:decease_scenarios} presents the value of the influence measure for all the features, in order to identify which ones are influencing the most.

  \begin{table}[H]
 	\begin{center}
 		\resizebox{10cm}{!}{
 		\begin{tabular}{|c|c|c|c|c|c|c||c|}
 			\hline
 			& \texttt{sex} & \texttt{age} &  \texttt{cardi} &  \texttt{resp} &  \texttt{meta} &  \texttt{uri} & Total \\
 			\hline
 			$\texttt{cardi}=2$ & 0.025 & \textbf{0.151} & \textbf{0.049} & 0.016 & 0.014 & 0.023 & 0.279\\ 
 			\hline
 			$\texttt{meta}=2$ & 0.035 & \textbf{0.142} & 0.034 & 0.062 & \textbf{0.078} & 0.017 & 0367\\
 			\hline 		
 		\end{tabular}
 	}
 	\caption{Influence measure. Decease = 1.}
 	\label{tab:decease_scenarios}
 \end{center}
 \end{table} \vspace{-0.5cm}
  From Table~\ref{tab:decease_scenarios} it can be seen that \texttt{age} is the most influential feature in these two scenarios, although the features chosen in each case (\texttt{cardi} and \texttt{meta}, respectively) are the second most influential.

 Figure~\ref{fig:icu} shows, surprisingly, a minor influence of age on ICU admissions. This is probably because in the first wave of COVID-19 in Spain, a considerable number of elderly died in nursing homes before they could even be hospitalised or admitted to ICU. In any case, \texttt{age} generates an influence scenario when it is worth $2$. As in Figure~\ref{fig:decease}, in the case of age the blue and red lines are very close, showing that the total influence in this particular situation is mainly due to age.
 
 Another influence scenario presented in Figure~\ref{fig:icu} is the one corresponding to the \texttt{meta} feature being equal to $2$. In that case, the blue and red lines are far apart, so we show in Table~\ref{tab:icu_scenarios} the value of the influence measure for all features. It can be observed that all features are influential, although the most influential are, in this order, age and metabolic diseases.
 
  \begin{table}[H]
 	\begin{center}
 		\resizebox{10cm}{!}{
 		\begin{tabular}{|c|c|c|c|c|c|c||c|}
 			\hline
 			& \texttt{sex} & \texttt{age} &  \texttt{cardi} &  \texttt{resp} &  \texttt{meta} &  \texttt{uri} & Total \\
 			\hline
 			$\texttt{meta}=2$ & 0.064 & \textbf{0.098} & 0.072 & 0.071 & \textbf{0.084} & 0.022 & 0.412\\ 
 			\hline 		
 		\end{tabular}
 	}
 	\caption{Influence measure. ICU admission = 1.}
 	\label{tab:icu_scenarios}
 \end{center}
 \end{table} \vspace{-0.5cm}
 Figure~\ref{fig:hos} allows us to identify other influence scenarios, among which we highlight those corresponding to \texttt{age} equal to $0$, \texttt{meta} equal to $2$ and \texttt{resp} equal to $2$. In this case, although the blue and red lines tend to coincide more in the \texttt{age} feature, they are considerably separated in all the influence scenarios. Therefore, we show in Table~\ref{tab:hos_scenarios} the value of the influence measure for all features in the three scenarios.

  \begin{table}[H]
	\begin{center}
		\resizebox{10cm}{!}{
		\begin{tabular}{|c|c|c|c|c|c|c||c|}
			\hline
			& \texttt{sex} & \texttt{age} &  \texttt{cardi} &  \texttt{resp} &  \texttt{meta} &  \texttt{uri} & Total \\
			\hline
			$\texttt{age}=0$ & -0.001 & \textbf{-0.184} & -0.013 & \textbf{-0.047} & -0.030 & -0.034 & -0.310\\ 
			\hline
			$\texttt{resp}=2$ & 0.039 & \textbf{0.085} & -0.009 & \textbf{0.113} & 0.014 & 0.005 & 0.247\\ 
			\hline
			$\texttt{meta}=2$ & 0.026 & \textbf{0.095} & 0.054 & 0.012 & \textbf{0.089} & -0.017 & 0.247\\
			\hline 		
		\end{tabular}
	}
	\caption{Influence measure. Need for hospitalisation = 1.}
	\label{tab:hos_scenarios}
\end{center}
  \end{table}\vspace{-0.5cm}
  Again, age remains a highly influential feature in the occurrence of hospitalisation in all the influence scenarios we have detected. In the first scenario, when \texttt{age} is $0$, what we observe is that the marked tendency towards less hospitalisation when patients are young is mainly due to their youth, although we also detect an important influence of good health in terms of respiratory ailments. In the influence scenario when $\texttt{resp}=2$, the measure indicates that respiratory diseases are the most influential in the need for hospitalisation, even more so than age. Somehow we detect that respiratory pathologies, in addition to age, are considerably influential in the need for hospitalisation of COVID-19 patients.
  
  In light of the above, it is evident that the most influential feature in all the response variables considered is age: young people are less likely to need hospitalisation and admission to the ICU, as well as to die from COVID-19; the only exception we detected is that elderly people who die have a tendency to die quickly, even before being admitted to the ICU. 
  
  With this in mind, we could look further for other influential features by eliminating the age effect. That is, we can remove \texttt{age} from the list of features (i.e., following the notation in Section~\ref{sec:theory},  $T=K\setminus \{2\}$, where $X_{2}=\texttt{age})$ and calculate the corresponding measure of influence. Through this approach, the expectation is that fewer influential scenarios will be detected; but in detected cases, the most influential features after age may come to light. We perform this analysis for the sub-sample in which we have the largest number of observations: the one corresponding to need for hospitalisation equal to 1.

Figure~\ref{fig:hos_feat_T} seems to confirm the considerable influence of respiratory diseases on the need for hospitalisation of COVID-19 patients. Indeed, the only positive influence scenario detected occurs when $\texttt{resp}=2$. Note also that, in this case, the blue and red lines are close, so that the total influence detected is mostly due to respiratory pathologies.

There is another scenario of influence when $\texttt{cardi}=0$. In this case, it is striking that the red line is close to the point $(0,0)$. This seems to indicate that in healthy individuals regarding cardiac functions an important influence on the decrease in hospitalisations is detected, but however such a decrease is not due to the feature \texttt{cardi}. To detect which is the most influential feature in this case, we show in Table~\ref{tab:no_age_scenarios} the value of the measure of influence when $\texttt{cardi}=0$ and any other of the pathologies considered is also $0$. Notice that in these three cases, feature \texttt{resp} is the most influential by far. Once again, the data we handle seem to confirm the important influence of the presence of respiratory pathologies on the need for hospitalisation of COVID-19 patients.

\noindent
\begin{figure}[H]
	\begin{minipage}[c]{12.5cm}
		\begin{center}
			\includegraphics[width=5.5cm]{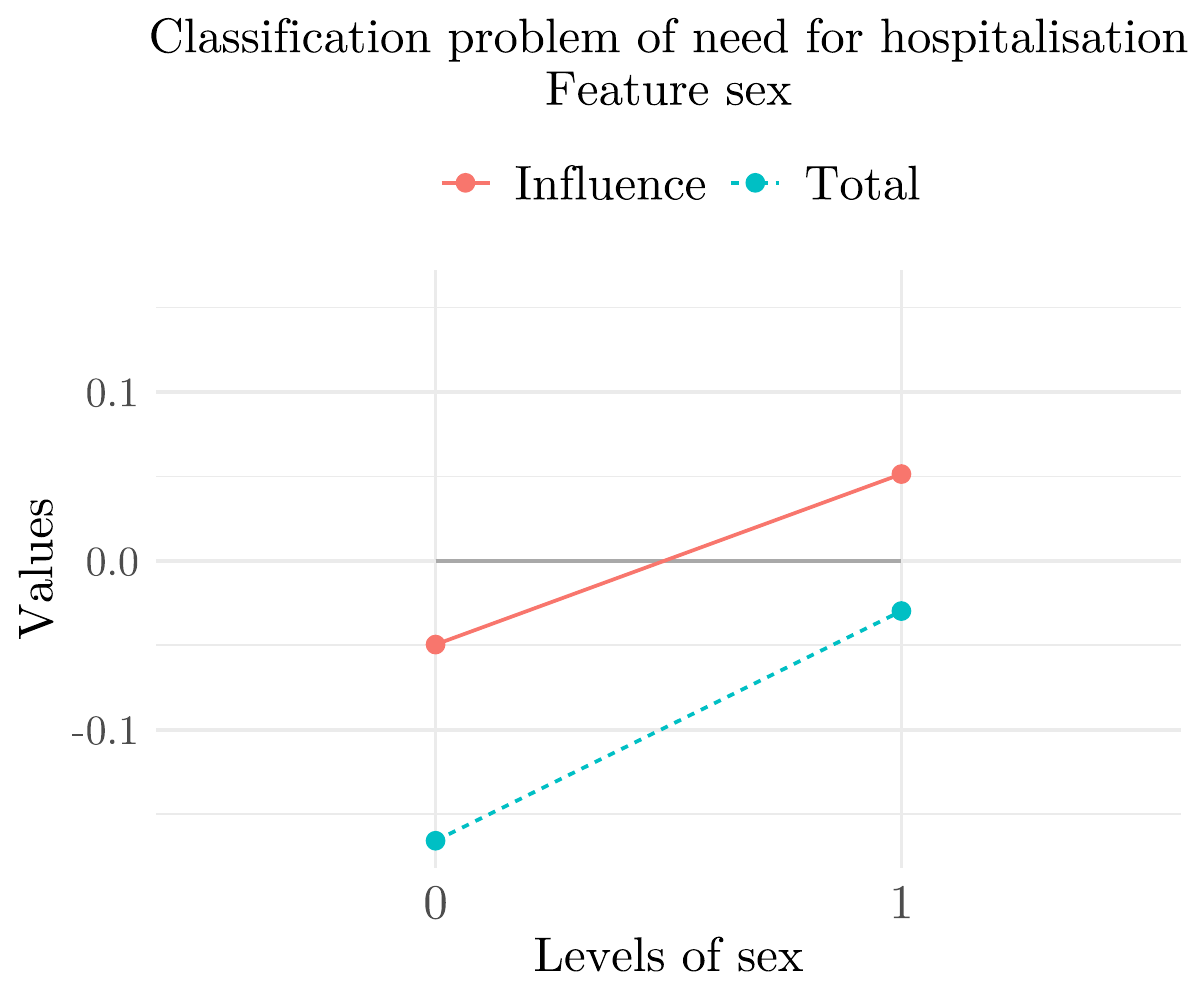}
		\end{center}
	\end{minipage}

	\noindent
	\begin{minipage}[c]{6.3cm}
		\begin{center}
			\includegraphics[width=5.5cm]{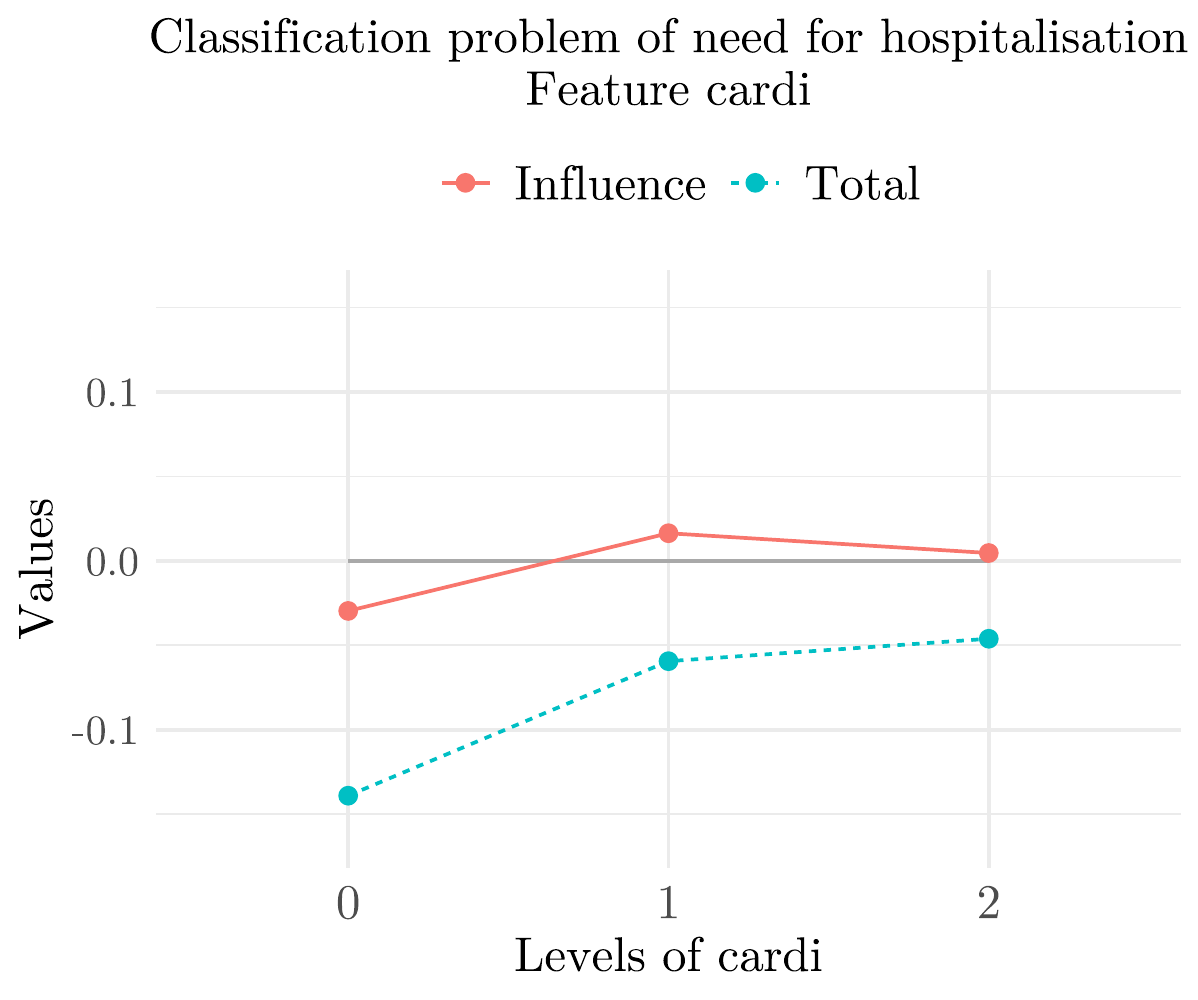}
		\end{center}
	\end{minipage}
	\hspace{0.2cm}
	\begin{minipage}[c]{6.3cm}
		\begin{center}
			\includegraphics[width=5.5cm]{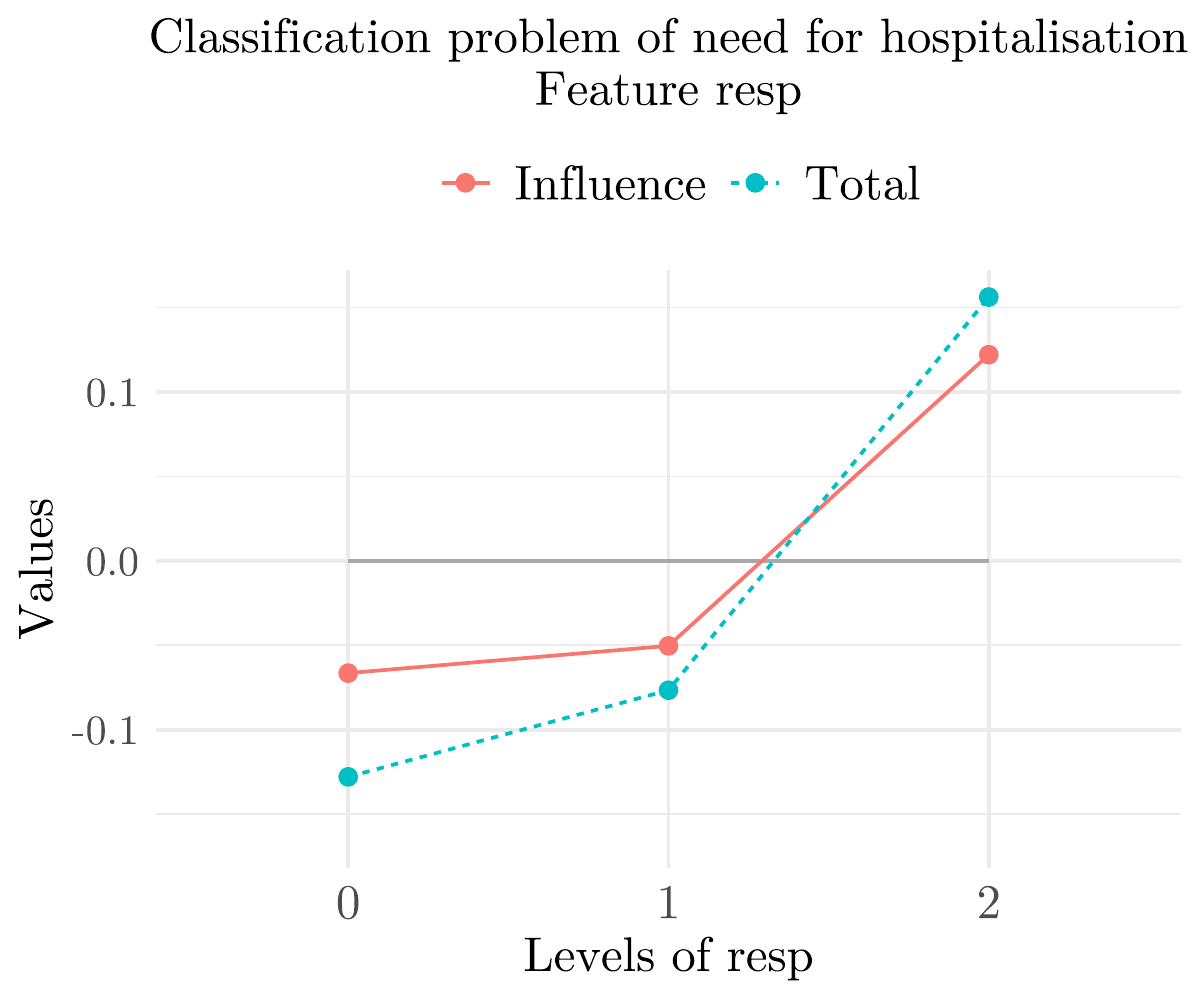}
		\end{center}
	\end{minipage}
	
	\noindent
	\begin{minipage}[c]{6.3cm}
		\begin{center}
			\includegraphics[width=5.5cm]{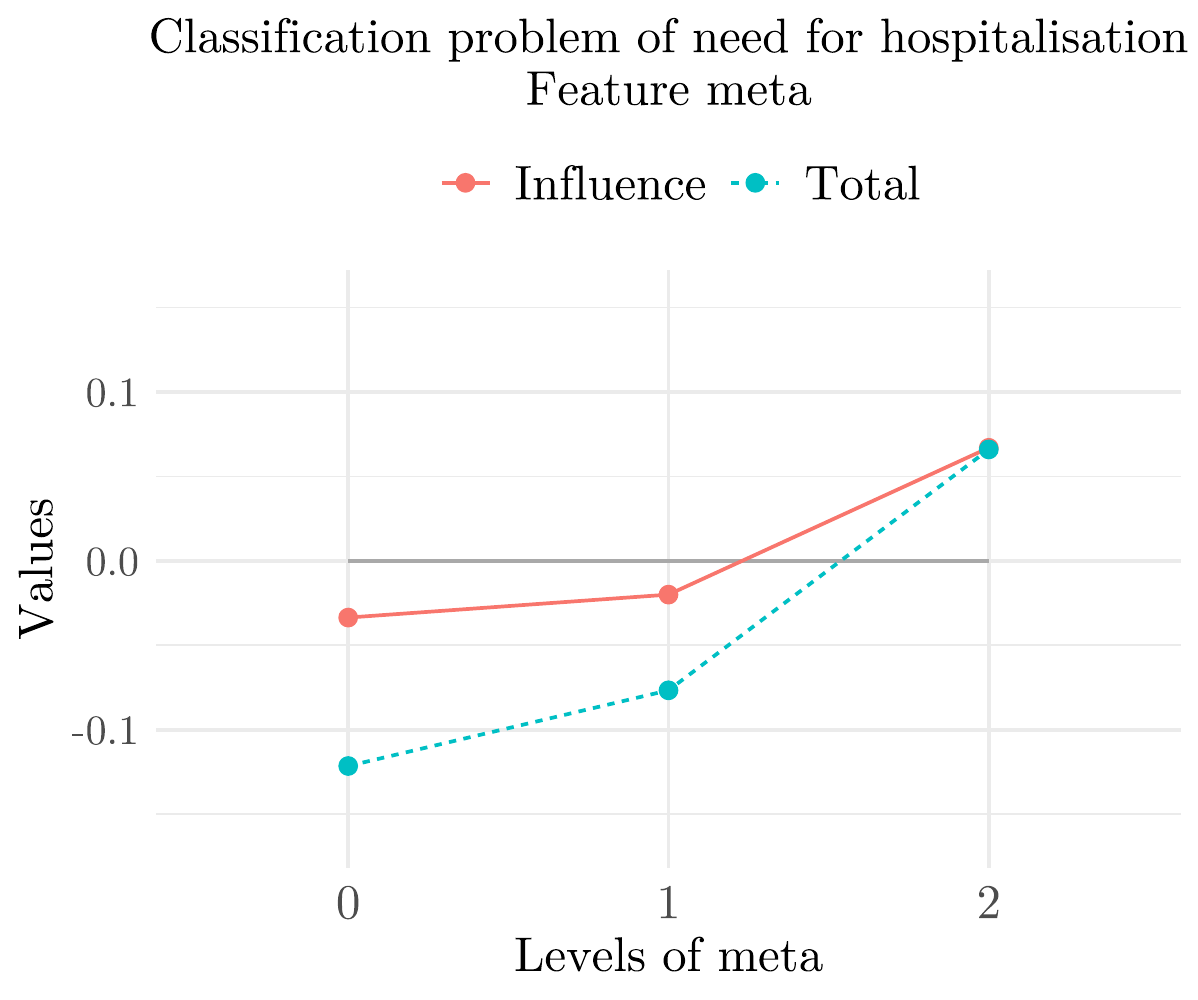}
		\end{center}
	\end{minipage}
	\hspace{0.1cm}
	\begin{minipage}[c]{6.3cm}
		\begin{center}
			\includegraphics[width=5.5cm]{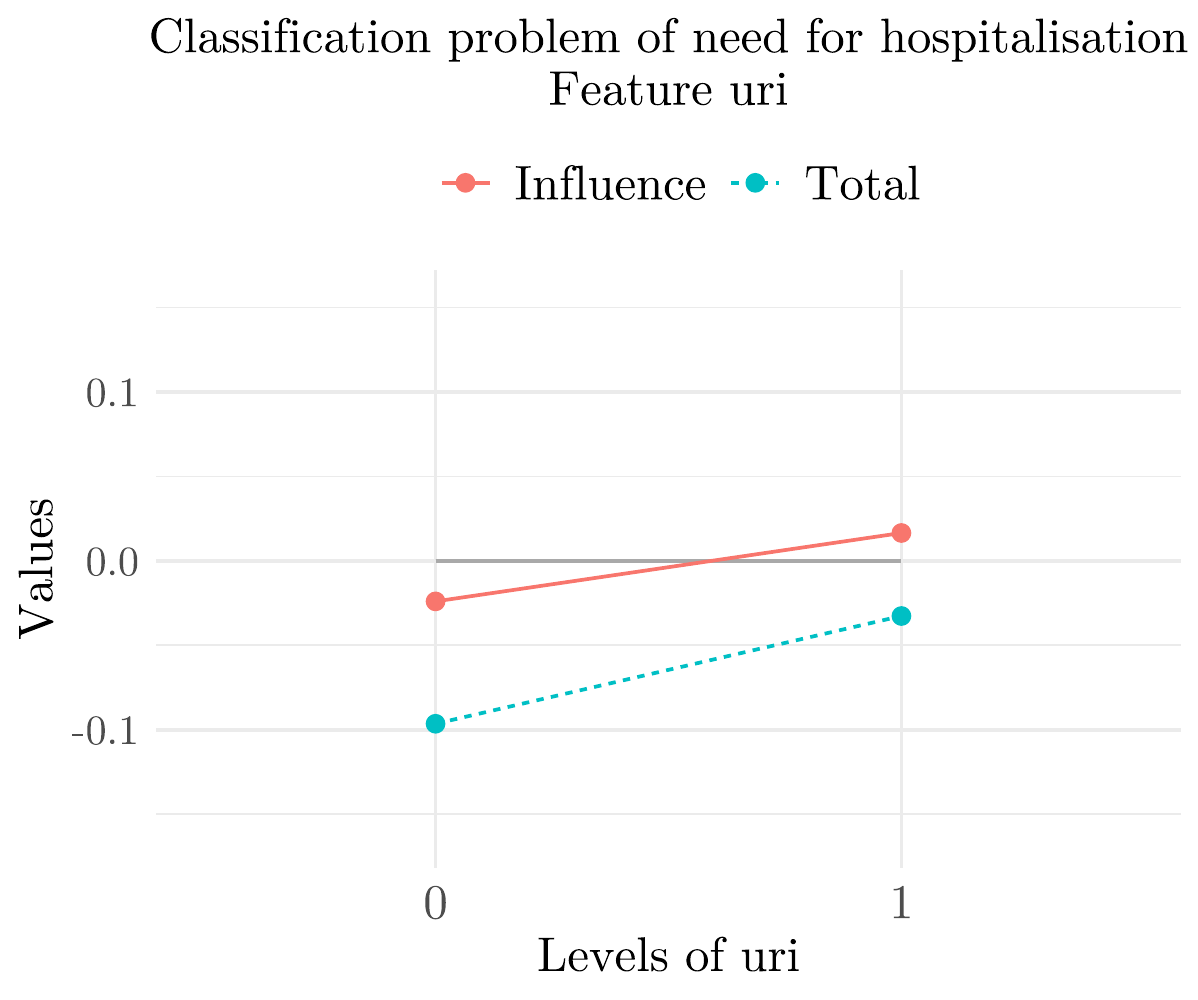}
		\end{center}
	\end{minipage}
	\caption{Influence and total influence for features $K\setminus \{2\}$ on the need for hospitalisation.}
	\label{fig:hos_feat_T}
\end{figure}

\begin{table}[H]
	\begin{center}
		\resizebox{10cm}{!}{
			\begin{tabular}{|c|c|c|c|c|c||c|}
				\hline
				& \texttt{sex} &   \texttt{cardi} &  \texttt{resp} &  \texttt{meta} &  \texttt{uri} & Total \\
				\hline
				$\texttt{cardi}=0$, $\texttt{resp}=0$ & 0.006 & -0.028 & \textbf{-0.063} & -0.025 & -0.046 &  -0.156\\ 
				\hline
				$\texttt{cardi}=0$, $\texttt{meta}=0$ & 0.009 & -0.038 & \textbf{-0.052} &  -0.033 & -0.037 & -0.151\\ 
				\hline
				$\texttt{cardi}=0$, $\texttt{uri}=0$ &  0.005  & -0.030 & \textbf{-0.054} &  -0.022 & -0.042 &  -0.143\\
				\hline 		
			\end{tabular}
		}
		\caption{Influence measure without considering \texttt{age}. Need for hospitalisation = 1.}
		\label{tab:no_age_scenarios}
	\end{center}
\end{table}\vspace{-0.7cm}

\section{Conclusions}\label{sec:conclusions}
This paper addresses and provides a methodological contribution to the important problem of classification, which is of great interest in machine learning. It introduces a new general measure of the influence that various features of a set of individuals have on their classification, that is, on the category or value they take for a given response variable. For the construction of such measure of influence, we consider several ideas taken from game theory. In particular, starting from the problem of measuring influence on classification, we define a cooperative game (whose players are the features considered) and apply a solution. This solution, known as the Shapley value, is closely connected with the idea of ``contribution'', and applied in this context to classification. Together with the definition of the measure of influence, an axiomatic characterisation theorem is stated and mathematically proved. The properties used in this result are adaptations of Shapley value's properties in the general context of cooperative games. The proposed adaptations yield highly desirable properties of the influence measure from the exploratory data analysis point of view. To test the scope and adequacy of the proposed influence measure, a control experiment  that provides a very satisfactory result is designed. Our proposal is also compared with the influence measure defined in \cite{Datta2015}, which also uses ideas from game theory.  Section~\ref{sec:covid} provides an application of our measure to the study of a Spanish database of patients infected with COVID-19 from the first wave of the pandemic, between March and May 2020. The aim of this application is to determine which demographic features, as well as previous pathologies, are the most influential in the classification of a patient regarding their potential need for hospitalisation, admission to the intensive care unit, or death. Initial results obtained present a promising future for the technique proposed here as a decision support tool, especially in the field of disease management. It serves, in particular, to alert medical professionals of the importance of certain patient characteristics, such as age or prior pathologies, as opposed to the lesser importance or influence of others. Such characteristics potentially pose an added difficulty in patients with a given disease, which should be taken into account both in the care and treatment that these patients should receive and in the planning of resources  destined for them.

As for future lines of research, we believe that additional work on the recently introduced measure of influence is worthwhile. We cite, for example, the desirability of further analysing the sensitivity of the results provided by the measure of influence according to the classifiers used. It would also be possible to complete the application presented using data from successive waves of the COVID-19 pandemic. In such case, it would be interesting to include a new variable distinguishing the virus strain, or even analyse the data separately depending on the type of strain, as it is known that new emerging strains behave differently.
 Finally, it may be appealing to further study the interest of this new measure of influence by exploring its relation with other statistical techniques of multivariate analysis, as well as extending it to continuous scenarios (for instance, considering that some of the features are continuous variables).

\section*{Acknowledgements}
The authors are grateful to Ricardo Cao Abad and to the {\it Direcci\'on Xeral de Sa\'ude P\'ublica} of the Xunta de Galicia in Spain. This work  has been supported by the ERDF, the Government of Spain/AEI [grants MTM2017-87197-C3-1-P and MTM2017-87197-C3-3-P]; the Xunta de Galicia [{\it Grupos de Referencia Competitiva}  ED431C-2016-015 and ED431C-2017/38, and {\it Centro Singular de Investigaci\'on de Galicia} ED431G/01]; and by the collaborative research project of the IMAT ``Mathematical, statistical and dynamic study of the epidemic COVID-19'', subsidized by the Vice-Rector's Office for Research and Innovation at the University of Santiago de Compostela, Spain. The research of Laura Davila-Pena has been funded by the Government of Spain [grant FPU17/02126]. We
would also like to thank the three anonymous referees and the editor for their constructive comments and suggestions,
which helped us to improve the final version of this paper.


\bibliography{mybibfile}

\begin{thebibliography}{12}
\expandafter\ifx\csname natexlab\endcsname\relax\def\natexlab#1{#1}\fi
\providecommand{\url}[1]{\texttt{#1}}
\providecommand{\href}[2]{#2}
\providecommand{\path}[1]{#1}
\providecommand{\DOIprefix}{doi:}
\providecommand{\ArXivprefix}{arXiv:}
\providecommand{\URLprefix}{URL: }
\providecommand{\Pubmedprefix}{pmid:}
\providecommand{\doi}[1]{\href{http://dx.doi.org/#1}{\path{#1}}}
\providecommand{\Pubmed}[1]{\href{pmid:#1}{\path{#1}}}
\providecommand{\bibinfo}[2]{#2}
\ifx\xfnm\relax \def\xfnm[#1]{\unskip,\space#1}\fi
\bibitem[{Algaba et~al.(2019)Algaba, Fragnelli \&
  S{\'a}nchez-Soriano}]{Algaba2019}
\bibinfo{author}{Algaba, E.}, \bibinfo{author}{Fragnelli, V.}, \&
  \bibinfo{author}{S{\'a}nchez-Soriano, J.} (\bibinfo{year}{2019}).
\newblock {\it \bibinfo{title}{Handbook of the Shapley value}\/}.
\newblock (\bibinfo{edition}{1st} ed.).
\newblock \bibinfo{publisher}{CRC Press, Taylor \& Francis}.
\bibitem[{Breiman(2001)}]{Breiman2001}
\bibinfo{author}{Breiman, L.} (\bibinfo{year}{2001}).
\newblock \bibinfo{title}{Random forests}.
\newblock {\it \bibinfo{journal}{Machine Learning}\/},  {\it
  \bibinfo{volume}{45}\/}, \bibinfo{pages}{5--32}.
\bibitem[{Datta et~al.(2015)Datta, Datta, Procaccia \& Zick}]{Datta2015}
\bibinfo{author}{Datta, A.}, \bibinfo{author}{Datta, A.},
  \bibinfo{author}{Procaccia, A.~D.}, \& \bibinfo{author}{Zick, Y.}
  (\bibinfo{year}{2015}).
\newblock \bibinfo{title}{Influence in classification via cooperative game
  theory}.
\newblock In {\it \bibinfo{booktitle}{Twenty-Fourth International Joint
  Conference on Artificial Intelligence}\/} (pp. \bibinfo{pages}{511--517}).
\bibitem[{Fern{\'a}ndez-Delgado et~al.(2014)Fern{\'a}ndez-Delgado, Cernadas,
  Barro \& Amorim}]{FernandezDelgado2014}
\bibinfo{author}{Fern{\'a}ndez-Delgado, M.}, \bibinfo{author}{Cernadas, E.},
  \bibinfo{author}{Barro, S.}, \& \bibinfo{author}{Amorim, D.}
  (\bibinfo{year}{2014}).
\newblock \bibinfo{title}{Do we need hundreds of classifiers to solve real
  world classification problems?}
\newblock {\it \bibinfo{journal}{Journal of Machine Learning Research}\/},
  {\it \bibinfo{volume}{15}\/}, \bibinfo{pages}{3133--3181}.
\bibitem[{Ghaddar \& Naoum-Sawaya(2018)}]{Ghaddar2018}
\bibinfo{author}{Ghaddar, B.}, \& \bibinfo{author}{Naoum-Sawaya, J.}
  (\bibinfo{year}{2018}).
\newblock \bibinfo{title}{High dimensional data classification and feature
  selection using support vector machines}.
\newblock {\it \bibinfo{journal}{European Journal of Operational Research}\/},
  {\it \bibinfo{volume}{265}\/}, \bibinfo{pages}{993--1004}.
\bibitem[{Gonz{\'a}lez-D{\'i}az et~al.(2010)Gonz{\'a}lez-D{\'i}az,
  Garc{\'i}a-Jurado \& Fiestras-Janeiro}]{Gonzalez2010}
\bibinfo{author}{Gonz{\'a}lez-D{\'i}az, J.},
  \bibinfo{author}{Garc{\'i}a-Jurado, I.}, \&
  \bibinfo{author}{Fiestras-Janeiro, M.~G.} (\bibinfo{year}{2010}).
\newblock {\it \bibinfo{title}{An introductory course on mathematical game
  theory}\/}.
\newblock (\bibinfo{edition}{1st} ed.).
\newblock \bibinfo{publisher}{American Mathematical Society}.
\bibitem[{Li \& Chen(2020)}]{Li2020}
\bibinfo{author}{Li, T.}, \& \bibinfo{author}{Chen, J.} (\bibinfo{year}{2020}).
\newblock \bibinfo{title}{Alliance formation in assembly systems with
  quality-improvement incentives}.
\newblock {\it \bibinfo{journal}{European Journal of Operational Research}\/},
  {\it \bibinfo{volume}{285}\/}, \bibinfo{pages}{931--940}.
\bibitem[{Liu et~al.(2020)Liu, Ji, Tang \& Li}]{Liu2020}
\bibinfo{author}{Liu, D.}, \bibinfo{author}{Ji, X.}, \bibinfo{author}{Tang,
  J.}, \& \bibinfo{author}{Li, H.} (\bibinfo{year}{2020}).
\newblock \bibinfo{title}{A fuzzy cooperative game theoretic approach for
  multinational water resource spatiotemporal allocation}.
\newblock {\it \bibinfo{journal}{European Journal of Operational Research}\/},
  {\it \bibinfo{volume}{282}\/}, \bibinfo{pages}{1025--1037}.
\bibitem[{Myerson(1980)}]{Myerson1980}
\bibinfo{author}{Myerson, R.~B.} (\bibinfo{year}{1980}).
\newblock \bibinfo{title}{Conference structures and fair allocation rules}.
\newblock {\it \bibinfo{journal}{International Journal of Game Theory}\/},
  {\it \bibinfo{volume}{9}\/}, \bibinfo{pages}{169--182}.
\bibitem[{Saavedra-Nieves \& Saavedra-Nieves(2020)}]{Saavedra2020}
\bibinfo{author}{Saavedra-Nieves, A.}, \& \bibinfo{author}{Saavedra-Nieves, P.}
  (\bibinfo{year}{2020}).
\newblock \bibinfo{title}{On systems of quotas from bankruptcy perspective: the
  sampling estimation of the random arrival rule}.
\newblock {\it \bibinfo{journal}{European Journal of Operational Research}\/},
  {\it \bibinfo{volume}{285}\/}, \bibinfo{pages}{655--669}.
\bibitem[{Shapley(1953)}]{Shapley1953}
\bibinfo{author}{Shapley, L.~S.} (\bibinfo{year}{1953}).
\newblock \bibinfo{title}{A value for n-person games}.
\newblock In \bibinfo{editor}{H.~W. Kuhn}, \& \bibinfo{editor}{A.~W. Tucker}
  (Eds.), {\it \bibinfo{booktitle}{Contributions to the Theory of Games
  ({AM}-28), Volume {II}}\/} (pp. \bibinfo{pages}{307--318}).
\newblock \bibinfo{publisher}{Princeton University Press}.
\bibitem[{Strumbelj \& Kononenko(2010)}]{Kononenko2010}
\bibinfo{author}{Strumbelj, E.}, \& \bibinfo{author}{Kononenko, I.}
  (\bibinfo{year}{2010}).
\newblock \bibinfo{title}{An efficient explanation of individual
  classifications using game theory}.
\newblock {\it \bibinfo{journal}{Journal of Machine Learning Research}\/},
  {\it \bibinfo{volume}{11}\/}, \bibinfo{pages}{1--18}.

\end{thebibliography}

\end{document}